\documentclass{article}
\pdfoutput=1

\usepackage[preprint]{neurips_2026}
\PassOptionsToPackage{numbers,sort&compress}{natbib}
\usepackage{neurips_2026}
\usepackage[utf8]{inputenc}
\usepackage[T1]{fontenc}
\usepackage{wrapfig}
\usepackage{graphicx}
\usepackage{booktabs}
\usepackage{float}
\usepackage[pagebackref,breaklinks,colorlinks]{hyperref}
\usepackage[ruled,linesnumbered]{algorithm2e}
\usepackage{url}
\usepackage{booktabs}
\usepackage{amsfonts}
\usepackage{nicefrac}
\usepackage{microtype}
\usepackage[table]{xcolor}
\usepackage{amsmath}
\usepackage{amssymb}
\usepackage{mathtools}
\usepackage{amsthm}
\usepackage{graphicx}
\usepackage{subcaption}
\usepackage{multirow}
\usepackage{bbm}
\usepackage[capitalize,noabbrev]{cleveref}
\usepackage{enumitem}

\theoremstyle{plain}
\newtheorem{theorem}{Theorem}[section]

\newtheorem{proposition}[theorem]{Proposition}
\newtheorem{lemma}[theorem]{Lemma}
\newtheorem{corollary}[theorem]{Corollary}
\theoremstyle{definition}
\newtheorem{definition}[theorem]{Definition}

\theoremstyle{remark}
\newtheorem{remark}[theorem]{Remark}
\newcolumntype{G}{>{\columncolor{gray!15}}c}
\definecolor{myblue}{RGB}{70,130,180}
\definecolor{mygreen}{RGB}{60,179,113}
\definecolor{myorange}{RGB}{255,140,0}
\definecolor{mygray}{RGB}{100,100,100}
\definecolor{backpropcolor}{RGB}{220,20,60}
\title{SMI: Statistical Membership Inference for Reliable Unlearned Model Auditing}
\author{
Jialong Sun\thanks{Equal contribution \texttt{jialong.sunh@gmail.com}.}\\
Shenzhen University of Advanced Technology
\And
Zeming Wei\footnotemark[1]\\
Peking University
\And
Jiaxuan Zou\footnotemark[1]\\
Xi'an Jiaotong University
\And
Jiacheng Gong\\
Heilongjiang University
\And
Jie Fu\\
Stevens Institute of Technology
\And
Chengyang Dong\\
Shenzhen University of Advanced Technology
\And
Heng Xu\\
Shenzhen University of Advanced Technology
\And
Jialong Li\\
Shenzhen University of Advanced Technology
\And
Bo Liu\thanks{Corresponding author: \texttt{liubo@suat-sz.edu.cn}}\\
Shenzhen University of Advanced Technology
}

\begin{document}

\maketitle

\begin{abstract}
Machine unlearning (MU) is essential for enforcing the right to be forgotten in machine learning systems. A key challenge of MU is how to reliably audit whether a model has truly forgotten specified training data. Membership Inference Attacks (MIAs) are widely used for unlearned model auditing, where samples that evade membership detection are regarded as successfully forgotten. We show this assumption is fundamentally flawed: failed membership inference does not imply true forgetting. We prove that unlearned samples occupy fundamentally different positions in the feature space than non-member samples, making this alignment bias unavoidable and unobservable, which leads to systematically optimistic evaluations of unlearning performance. Meanwhile, training shadow models for MIA incurs substantial computational overhead. To address both limitations, we propose Statistical Membership Inference (SMI), a training-free auditing framework that reformulates auditing as estimating the non-member mixture proportion in the unlearned feature distribution. Beyond estimating the forgetting rate, SMI also provides bootstrap reference ranges for quantified auditing reliability. Extensive experiments show that SMI consistently outperforms all MIA-based baselines, with no shadow model training required. Overall, SMI establishes a principled and efficient alternative to MIA-based auditing methods, with both theoretical guarantees and strong empirical performance.
\end{abstract}
\section{Introduction}

Machine learning (ML) has been widely deployed in various critical domains, such as healthcare, finance, and Transportation, whose training processes may heavily depend on privacy-sensitive data under certain protections~\cite{Jeffsafty,Liusafty}. However, the authentication of data usage from users is not necessarily unwarranted, particularly in the background of policies like GDPR that propose ``the right to be forgotten'', which indicates users reserve the right to withdraw the data authentication with considerations like privacy-protection, ethical issues, or policy requirements~\cite{unlearningN,unlearningS,Yisafty,newunlearning}. This explicitly indicates an urgent need for an accountable mechanism that can timely remove the impact of related data on model behaviors when authentication is revoked, as well as restrictions on training, inference, or communication with these data for future model deployments~\cite{unlearningnunix}.

Given this background, machine unlearning has emerged as a prominent approach to address these requirements~\cite{unlearningicml1,unlearningicml2,wu2025reliable}. Generally, machine unlearning involves using post-training methods to eliminate the influence of specific training data on the trained model. Despite notable success, a key open question in machine unlearning remains: \emph{How can we effectively
evaluate the forgetting of a given set of data~\cite{nounealrning,nounlearning1}?}
In this paper, we refer to this evaluation as the \emph{auditing} problem of machine unlearning. Currently, common auditing criteria can be categorized into three types of backdoor attacks~\cite{back1,backdoor}, adversarial attacks~\cite{duikang1,duikang}, and membership inference attacks (MIA)~\cite{MIA,MIA2}.

Both backdoor and adversarial attacks attempt to poison the training data, leading to compromised model performance on specific tasks. Since machine unlearning can be viewed as recovering the model's ability by locating the poisoning data and unlearning them, these two kinds of attacks calculate the attack success rate  on the unlearned model for auditing~\cite{backandduikang}. Thus, a critical concern emerges: as defense techniques (e.g., adversarial training against adversarial attacks) can also decrease ASRs, yet do not achieve the unlearned goal, the trustworthiness of such auditing methods is compromised.
In addition, both attack strategies require training a new target model, which may be perceived as tampering with the underlying task.

Besides, MIA aims to judge whether a specific data point is involved in the model training process, determined by assessing the output of the model given this input~\cite{miaCCS}. Intuitively, MIA can be mathematically formulated as a threshold-based classification problem:
\begin{equation}
    \text{MIA}(x;w)=\mathbbm{1}_{\text{scoreMIA}(x;w)\ge \beta},
\end{equation}
where $\text{scoreMIA}(x;w)$ is the membership score for sample $x$ under model parameter $w$. So far, a thread of MIA-based methods has paved the way for effectively auditing machine unlearning results~\cite{Miawhite}. Many variants have been implemented to improve MIA performance on diverse unlearning settings, such as  RMIA~\cite{RMIA}, IAM~\cite{IAM}, and RULI~\cite{newmia1}.
These attack-based auditing methods are prone to a failure mode in which failed attacks give rise to illusory forgetting.
Specifically, a failure in the attack strategy may be misinterpreted as the model having unlearned the sample. Consequently, this causes the auditing of the unlearning task to fail.

In this study, we aim to address the issue of illusory forgetting caused by failed membership inference attacks.
We revisit the validity of the MIA paradigm.
In an MIA task, the goal is essentially to train a binary classifier that distinguishes member sample features $F_m$ (e.g., features from the training dataset) from non-member sample features $F_n$ (e.g., features from the test dataset).
However, in a real unlearning task, the audited targets are the feature representations of unlearned samples $F_u$~\cite{Sha1,Sha2,Sha3}. Conventional MIA assumes that the audit features of $F_u$ are close to those of $F_n$, and therefore treats samples whose features in $F_u$ are classified as $F_n$ as successfully unlearned. As illustrated in Figure~\ref{fig:tu1}, this assumption is not well aligned with the actual objective of unlearning.
 This is because, in the unlearning task, $F_m$ corresponds to feature representations that contribute positive gradients to the model, $F_n$ corresponds to feature representations that contribute no gradient, while $F_u$ corresponds to feature representations that induce negative gradients. Such a gradient-based distinction is fundamentally misaligned with the core assumption of MIA, which does not explicitly account for these opposing optimization dynamics.
Although some studies have attempted to improve the alignment between MIA and unlearning through group-level attack strategies or independent modeling procedures, these methods do not fundamentally resolve the alignment bias and still require training tens of shadow models, each as costly as the audited model itself, making the computational overhead prohibitive.
\begin{figure}
    \centering
    \includegraphics[width=1\linewidth]{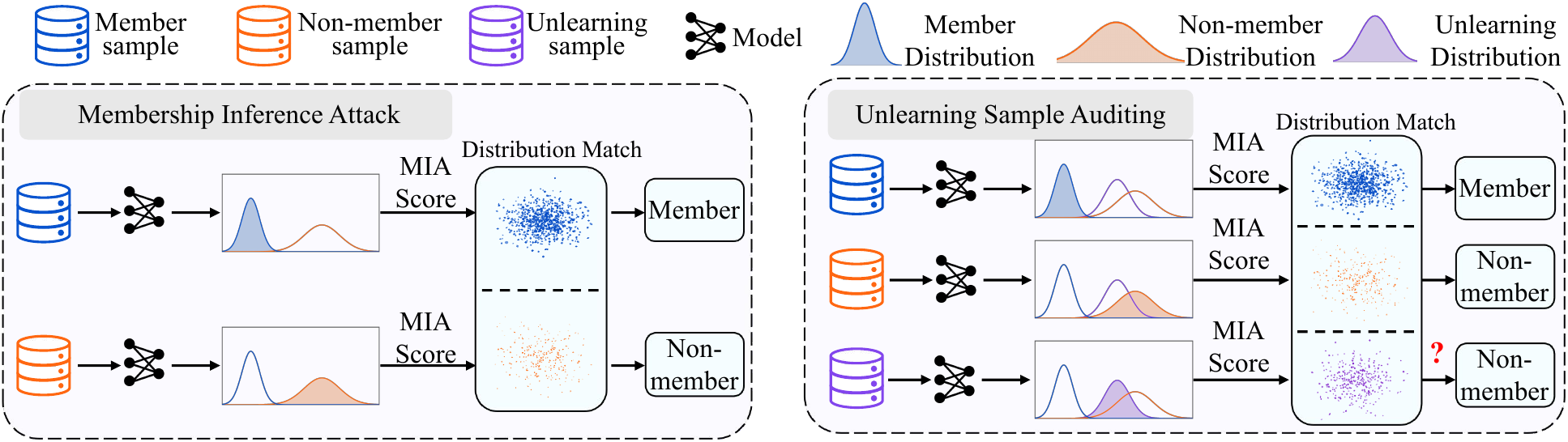}
    \caption{Misalignment Between MIA and the Unlearned Model Auditing Process}
    \label{fig:tu1}
\end{figure}

 Based on the above observations, we first conduct a theoretical analysis from two aspects. First, we formally prove that MIA-based strategies for auditing machine unlearning suffer from inevitable alignment errors. Second, based on the sample features migration property during the unlearning process, we introduce optimal transport theory to derive the corresponding optimal estimation. Building on this analysis, we propose a training-free statistical method for auditing unlearning tasks, called Statistical Membership Inference (SMI). Unlike existing methods that rely on shadow model training and classifier training, SMI does not require any shadow models or MIA classifiers. Instead, it only uses statistical techniques to estimate the mixture proportion between member and non-member distributions. In addition, SMI provides reference ranges for the audit results, which offer a reliable reference for audit decisions. Overall, our contributions are summarized as follows:

\begin{itemize}
    \item \textbf{Conceptually}, we show that failed MIA does not imply true forgetting, since unlearned samples may reside in a feature distribution fundamentally distinct from that of genuine non-members. This reveals a structural mismatch between MIA-based auditing and the goal of unlearned model auditing.

    \item \textbf{Theoretically}, we prove that this mismatch arises from an unavoidable yet unobservable alignment bias between unlearned samples and non-members. We further reformulate unlearned model auditing as estimating the member and non-member mixture proportion, shifting the Monge problem of unlearning from the parameter space to the feature space.

    \item \textbf{Practically}, we propose SMI and its variant SMI-M, a training-free auditing framework that estimates the forgetting rate without shadow models or learned classifiers, while providing bootstrap reference ranges to quantify auditing reliability.
\end{itemize}

\section{Revisiting MIA: From Attack to Auditing}
This section explains the limitations of MIA in unlearned model auditing and motivates the proposed improvement, SMI.  First, we use PAC-Bayes theory to derive the alignment bias in unlearned model auditing. Second, we adopt an optimal-transport perspective to revise the auditing strategy.

\textbf{Alignment Bias in MIA.}
We decompose the attacker’s learning process to identify the alignment bias in MIA-based unlearned model auditing. Let $F_n$, $F_m$, and $F_u$ be the auditing feature sets of the retained member dataset $D_n$, non-member dataset $D_m$, and unlearning dataset $D_u$, and let $\mathcal F_n$, $\mathcal F_m$, and $\mathcal F_u$ be the corresponding empirical distributions~\cite{Unleark,MIAD}. Let $\mathcal P$ and $\mathcal Q$ denote the prior and posterior distributions over the attacker’s hypothesis space~\cite{lilun1,lilun2}. The $\chi^2$-based PAC-Bayes risk bound is stated in Theorem~\ref{2.1}.

\begin{theorem}[PAC-Bayesian Bound for MIA]\label{2.1}
    For a binary classifier, suppose that the attacker model $\mathcal A$ learns from a prior distribution $\mathcal P$ and induces a posterior distribution $\mathcal Q$. For a sample size $m$, with probability at least $1-\delta$, the following inequality holds for any distribution $\mathcal Q$:
    \begin{equation}
        R_D(\mathcal Q)
        \leq
        \underbrace{R_S(\mathcal Q)}_{\text{Empirical Risk}}
        +
        \underbrace{
        \sqrt{\chi^2(\mathcal Q \| \mathcal P)+1}
        \cdot
        \sqrt{\frac{2}{\delta m}}
        }_{\text{Statistical Error}},
    \end{equation}
where $R_D(\mathcal Q)$ is the true risk when hypotheses are sampled according to $\mathcal Q$, $R_S(\mathcal Q)$ is the average empirical risk computed on the training set, and $\chi^2(\mathcal Q\|\mathcal P)$ is the $\chi^2$ divergence between $\mathcal Q$ and $\mathcal P$.
\end{theorem}

Remark 1: We use a $\chi^2$-based PAC-Bayes framework because it provides convenient second-moment control for deriving the subsequent distribution-shift term.
Detailed proof refers to Appendix~\ref{T2.1}.

Theorem~\ref{2.1} gives a general decomposition of the attacker’s learning error. More importantly, it does not reveal the alignment discrepancy in unlearned model auditing. The reason is that the bound is derived on the MIA training distribution $\mathcal F_t$, which is constructed from $\mathcal F_n$ and $\mathcal F_m$, while the actual auditing target is $\mathcal F_u$. Conventional MIA implicitly treats $\mathcal F_t$ and $\mathcal F_u$ as aligned, which can introduce an alignment bias. We therefore incorporate the discrepancy between $\mathcal F_t$ and $\mathcal F_u$ into the PAC-Bayes bound and obtain a refined auditing-error decomposition in Corollary~\ref{coro}.

\begin{corollary}[PAC-Bayesian Bound for MIA Auditing Results]
    When the training data distribution $\mathcal F_t$ of the MIA attacker model $\mathcal A$ is used to audit the data distribution $\mathcal F_u$, for a sample size $m$, with probability at least $1-\delta$, the following risk inequality holds:
    \begin{equation}
        R_{\mathcal F_u}(\mathcal Q)
        \leq
        R_S(\mathcal Q)
        +
        \underbrace{
        \sqrt{
        \frac{2}{m\delta}
        \left(\chi^2(\mathcal Q \| \mathcal P)+1\right)
        }
        }_{\text{Statistical Error}}
        +
        \underbrace{
        \sqrt{
        \frac{1}{2}D_\infty(\mathcal F_t \| \mathcal F_u)
        }
        }_{\text{Auditing Error}} .
    \end{equation}
\label{coro}
where $D_\infty(\mathcal F_t \| \mathcal F_u)=\log\left(\sup \frac{\mathcal F_t}{\mathcal F_u}\right)$ is the Rényi divergence~\cite{duliang3}. Please refer to Appendix~\ref{C2.2} for a detailed explanation of the notation and all proofs.
\end{corollary}
This corollary shows that the MIA risk bound contains an additional auditing term controlled by $D_\infty(\mathcal F_t\|\mathcal F_u)$. This term increases with the distribution shift between $\mathcal F_t$ and $\mathcal F_u$, indicating that conventional MIA faces an alignment bias in unlearned model auditing. Thus, the auditing error is not only caused by empirical risk or statistical complexity, but also by the mismatch between the attacker’s training distribution and the actual auditing distribution.

\textbf{Revising the Auditing Strategy.}
To address the alignment bias in conventional MIA, we revise the auditing criterion from the goal of machine unlearning. After sufficient unlearning, the post-unlearning model should no longer preserve membership traces of the unlearned samples. Thus, auditing should not only test whether the attacker fails to identify their membership. It should also examine whether their auditing features are close to those of non-member samples.

From a transport perspective, machine unlearning can be viewed as moving the original model state toward the retrained model state. Let $\alpha_w$ be the parameter distribution after original training, and let $\beta_{w\setminus u}$ be the ideal retraining distribution after removing $D_u$. The ideal unlearning objective can then be written as the following Monge-type problem.
\begin{definition}
The Monge problem in parameter space is:
\begin{equation}
\min_{T_w:T_{w\#}\alpha_w=\beta_{w\setminus u}}
\int c_w\bigl(w,T_w(w)\bigr)\,d\alpha_w(w),
\end{equation}
where $T_w$ is the unlearning map in the parameter space, $T_{w\#}\alpha_w$ denotes the pushforward distribution obtained by transporting the original parameter distribution $\alpha_w$ through $T_w$, and $c_w$ is the parameter transport cost. This objective requires the post-unlearning model distribution to be aligned with the ideal retrained model distribution.
\end{definition}
\begin{figure}[t]
    \centering
    \includegraphics[width=1\linewidth]{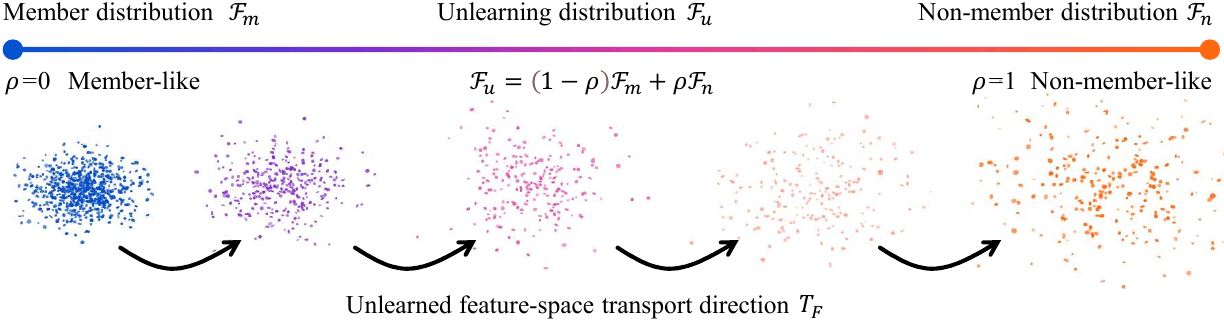}
    \caption{Approximate Representation of the Distribution Transport Problem as a Distribution Mixture Problem}
    \label{fig:tu2}
\end{figure}
Directly verifying this parameter-space objective is usually infeasible. Modern neural networks are high-dimensional, and different parameter settings can yield similar functions. More importantly, unlearned model auditing concerns whether unlearned samples still show membership traces in outputs or representations, instead of whether the parameters match a target distribution. Therefore, we move the auditing objective to the feature space. Instead of verifying the parameter distribution, we check whether $\mathcal F_u$ moves away from $\mathcal F_m$ and approaches $\mathcal F_n$. The feature-space auditing objective can then be written as the following Monge-type problem.
\begin{definition}
The Monge problem in feature space is:
\begin{equation}
\min_{T_F:T_{F\#}\mathcal F_u=\mathcal F_n}
\int c_F\bigl(z,T_F(z)\bigr)\,d\mathcal F_u(z),
\end{equation}
where $T_F$ is the transport map in the auditing feature space, $T_{F\#}\mathcal F_u$ is the pushforward of $\mathcal F_u$ under $T_F$, and $c_F$ measures the cost of transporting $\mathcal F_u$ to $\mathcal F_n$.
\end{definition}
The Monge formulation does not require explicitly solving $T_F$. It only specifies the desired auditing direction: after sufficient unlearning, $\mathcal F_u$ should move away from $\mathcal F_m$ and approach $\mathcal F_n$. Thus, auditing only needs to measure how far $\mathcal F_u$ has moved toward the non-member state. As illustrated in Figure~\ref{fig:tu2}, we therefore relax the transport objective into a mixture-proportion problem and approximate $\mathcal F_u$ as a convex combination of $\mathcal F_m$ and $\mathcal F_n$. The relaxation of the conditions in this Monge problem is discussed in detail in Appendix~\ref{diss}.
\begin{equation}
    \mathcal F_u
    \approx
    (1-\rho)\mathcal F_m+\rho\mathcal F_n,
    \qquad \rho\in[0,1],
\end{equation}
where $\rho$ measures the non-member-like proportion in $\mathcal F_u$. Thus, the new unlearned model auditing objective is to estimate the mixture coefficient $\rho$ instead of solving the feature-space transport map. In practice, $\rho$ can be estimated by minimizing the discrepancy between empirical distributions.
\begin{definition}
The practical estimation problem for the mixture proportion is defined as:
\begin{equation}
    \rho^*
    =
    \arg\min_{\rho\in[0,1]}
    D\left(
    \mathcal F_u,
    (1-\rho)\mathcal F_m+\rho\mathcal F_n
    \right), \label{7}
\end{equation}
where $D(\cdot,\cdot)$ denotes a distributional discrepancy measure. We then define the auditing objective as solving the optimization problem in Eq.~\eqref{7}, and use the resulting $\rho^*$ as the auditing metric for the unlearned model.
\end{definition}
\section{Method}
In this section, we propose Statistical Membership Inference (SMI) for reliable unlearned model auditing. Motivated by the previous analysis, SMI solves the optimization problem in Eq.~\eqref{7}. Statistical metrics for distributional distance are classical; see, e.g.,~\cite{duliang1,duliang2,duliang3}. Appendices~\ref{theory} and~\ref{mmd} provide the omitted mathematical details. We next introduce two solution strategies for SMI:
\begin{itemize}
\item \textbf{SMI:} An auditing tool that calculates low-order moments of distributions.
\item \textbf{SMI-M:} An auditing tool that calculates the Maximum Mean Discrepancy (MMD) between distributions.
\end{itemize}
\subsection{Statistical Membership Inference (SMI)}
For the mixture-ratio estimation problem, SMI estimates the mixing proportion using the low-order statistics of the three distributions $\mathcal F_m$, $\mathcal F_n$, and $\mathcal F_u$. Under the assumption in Eq.~\eqref{7}, the mixed auditing distribution can be expressed as follows.
\begin{proposition}[See Appendix~\ref{P3.1} for the proof.]\label{P1}
If $\mathcal F_u$ is modeled as a mixture of $\mathcal F_n$ and $\mathcal F_m$ with proportions $\rho$ and $1-\rho$, respectively, i.e.,
\begin{equation}
    \mathcal F_u=\rho \mathcal F_n +(1-\rho)\mathcal F_m,
\end{equation}
then its mean $\mu_u$ and covariance $\Sigma_u$ satisfy
\begin{equation}
        \mu_u=\rho \mu_n+(1-\rho)\mu_m,\quad
        \Sigma_u=\rho \Sigma_n+(1-\rho)\Sigma_m
        +(\rho-\rho^2)(\mu_n-\mu_m)(\mu_n-\mu_m)^\top.
\end{equation}
\end{proposition}
From Proposition~\ref{P1}, estimating the forgetting degree of $\mathcal F_u$ can be reduced to estimating the mixture coefficient $\rho$. SMI estimates $\rho$ by matching the empirical covariance of $\mathcal F_u$ with the covariance predicted by the mixture model.
\begin{lemma}[Solution of SMI]
Let
$d=\mu_n-\mu_m,B=dd^\top,
A=\Sigma_u-\Sigma_m,
G=\Sigma_n-\Sigma_m+B .$
SMI estimates the mixture ratio by solving:
\begin{equation}
\min_{0\le\rho\le 1}\left\|
\Sigma_u-
\left[
\rho \Sigma_n+(1-\rho)\Sigma_m
+(\rho-\rho^2)(\mu_n-\mu_m)(\mu_n-\mu_m)^\top
\right]
\right\|^2_F. \label{youhauwentismi}
\end{equation}
The solution is selected from the search space $\mathcal C$:
\begin{equation}
\rho^\star=
\arg\min_{\rho\in\mathcal C}
\left\|A-\rho G+\rho^2 B\right\|_F^2. \label{SMI-J}
\end{equation}
The search space $\mathcal C$ is:
\begin{equation}
\mathcal C
=
\{0,1\}
\cup
\Big\{
\rho\in[0,1]:
2\|B\|_F^2\rho^3
-3\langle G,B\rangle_F\rho^2 +
\left(\|G\|_F^2+2\langle A,B\rangle_F\right)\rho
-\langle A,G\rangle_F
=0
\Big\},
\end{equation}
where \(\langle A,B\rangle_F=\mathrm{tr}(A^\top B)\) denotes the Frobenius inner product. Specifically, the solution process can be described as first solving a univariate cubic equation in $\rho$, and then substituting the obtained roots into the search space to select the minimizer. The complete proof is provided in Appendix~\ref{L3.2}. \label{lemmaSMI}
\end{lemma}
\begin{wrapfigure}{r}{0.49\textwidth}
\noindent\makebox[\linewidth][r]{%
\begin{minipage}{0.47\textwidth}
\begin{algorithm}[H]
\caption{SMI}
\label{smi}
\KwIn{Member data $D_m$, non-member data $D_n$, pending-audit data $D_u$, model $w$, audit function $g(\cdot,w)$}
\KwOut{Mixture ratio $\rho^*$.}
    $(\mu_m, \Sigma_m) \leftarrow \operatorname{stats}(g(D_m;w))$\;
    $(\mu_n, \Sigma_n) \leftarrow \operatorname{stats}(g(D_n;w))$\;
    $(\Sigma_u) \leftarrow \operatorname{stats}(g(D_u;w))$\;
    $\displaystyle
    \rho^\star\leftarrow
\arg\min_{\rho\in\mathcal C}
\left\|A-\rho G+\rho^2 B\right\|_F^2$\;
\Return{$\rho^* $}\;
\end{algorithm}
\end{minipage}}
\vspace{-1.0em}
\end{wrapfigure}
The SMI optimization is simple, and the resulting $\rho^*$ serves as the final auditing result. SMI computes the mean and variance of neural features, rather than raw data. For a sample $x$ and model parameters $w$, we denote the extracted feature by $g(x;w)$. Algorithm~\ref{smi} gives the pseudo-code of SMI. The optimization in Lemma~\ref{lemmaSMI} may be unstable because it depends on the empirical statistics $\Sigma_u$, $\Sigma_n$, and $\Sigma_m$. With limited auditing samples, covariance estimates can fluctuate, especially for high-dimensional neural features. Direct covariance matching may therefore give unstable mixture-ratio estimates. To improve stability, we instead use kernel mean embeddings and match the embedding of $\mathcal F_u$ to a convex combination of those of $\mathcal F_m$ and $\mathcal F_n$.

\subsection{Statistical Membership Inference-MMD (SMI-M)}
The computational error in estimating the mixture ratio using low-order moments mainly arises from the estimation of the covariance matrix.
Therefore, we introduce a statistical feature optimization method, embedding $g(\cdot;w)$ into a RKHS, thereby transforming the original second-order moment (covariance $\Sigma$) optimization problem into a first-order moment embedding problem in RKHS.

Specifically, we introduce a kernel function $k(\cdot, \cdot)$ for $g(\cdot;w)$ and define the kernel mean embedding as $\mu^{(k)}_\cdot=\mathbb E[k(g(\cdot;w),g(\cdot;w)^\top)]$. This mapping transforms the second-order information of the original distribution into the first-order information of the embedded distribution. Accordingly, the Lemma~\ref{lemmaSMI} can be reformulated as follows:
\begin{lemma}[Optimization problem of SMI-M]
    For the embedded $\mu^{(k)}_f, \mu^{(k)}_v, \mu^{(k)}_t$, the optimization problem~\eqref{youhauwentismi} is equivalent to the following convex quadratic programming problem:
    \begin{equation}
        \min_{0\le\rho\le 1}\,\, \rho^2\|\mu^{(k)}_v-\mu^{(k)}_t \|^2_\mathcal H-2\rho\langle\mu_f^{(k)}-\mu_t^{(k)}, \mu_v^{(k)}-\mu_t^{(k)}\rangle_\mathcal H,\label{youhuamd}
    \end{equation}
where  $\|\cdot\|_\mathcal H$ denotes the Hilbert space norm, and $\langle\cdot,\cdot \rangle_{\mathcal H}$ denotes the Hilbert space inner product. This problem is a convex quadratic programming problem and can be solved directly by a solver without manually deriving a closed-form solution. The complete proof is provided in Appendix~\ref{A:3.3}.
\end{lemma}
\subsection{Resampling-based Uncertainty Estimation for SMI}
Inspired by shadow models in MIA, we use resampling to estimate the uncertainty of SMI. Shadow-model methods train multiple auxiliary models, which is costly for unlearned model auditing. SMI avoids this cost by bootstrapping $D_u$ in the auditing feature space. For each resampled set, we solve the SMI optimization again and obtain $\{\rho_b^*\}_{b=1}^{B}$.

Resampling doesn't change the point estimate $\rho^*$. It only provides an empirical reference range for finite-sample sensitivity. Concentrated $\rho_b^*$ indicates stable auditing, while large variation suggests that $\rho^*$ should be interpreted with caution. We report $\{\rho^*_{5\%},\rho^*_{95\%}\}$ as the reference range.

\section{Experiment}
To verify the effectiveness of SMI, we will examine it from the following five perspectives:
\begin{itemize}
\item[1.] \textbf{Performance:} \textit{Does SMI outperform conventional MIA-based auditing?}In Section 4.2, we evaluate SMI on public datasets and the unlearning-audit benchmark, and compare it with conventional MIA methods as baselines.
\item[2.] \textbf{Computational Cost:} \textit{Does SMI avoid the substantial computational overhead of shadow-model-based auditing?}
In Section 4.3, we compare the runtime and inference cost of different SMI variants and conventional MIA methods, with a particular focus on shadow model and non-shadow model settings.
\item[3.] \textbf{Limited-sample Reliability:} \textit{Can SMI provide stable auditing results with limited unlearning samples?}
In Section 4.4, we evaluate whether SMI remains reliable when the number of audited samples is limited, and analyze the stability of its estimated forgetting ratio and resampling-based reference range.
\item[4.] \textbf{Robustness:} \textit{Is SMI-M stable under different experimental conditions?}
In Section 4.5, we conduct robustness experiments for SMI-M to assess whether its auditing performance remains stable under variations in data, model behavior, or sampling conditions.
\item[5.] \textbf{Practicality:} \textit{Does SMI work in real unlearning scenarios?}
In Section 4.6, we evaluate SMI in realistic machine unlearning settings to examine whether it can reliably audit actual forgetting behavior beyond controlled benchmark comparisons.
\end{itemize}
\subsection{Experimental Setup}
\textbf{Dataset and Models:} We construct image-classification tasks using the CIFAR-10, CIFAR-100~\cite{Cifar}, and CINIC-10~\cite{cinic} datasets. We adopt ResNet-18~\cite{renset18}, Resnet50~\cite{RESNET50} and VIT~\cite{vit} for both the target model and the shadow models required by conventional MIA baselines.

\textbf{Baselines:} We evaluate model-agnostic MIA methods using the unlearning audit benchmark released alongside IAM~\cite{IAM}. We adopt seven MIA-based strategies as baselines forunlearned model auditing, as listed below:
\begin{itemize}
\item \textbf{Random:} Randomly estimates the member proportion as a worst-case reference.
\item \textbf{Unlea:}This metric measures the output discrepancy before and after unlearning.
    \item \textbf{LiRA:} Performs auditing by comparing the output likelihood ratio between the unlearning dataset and the training dataset on the target model.
    \item \textbf{EMIA:} Extends LiRA with incremental auditing decisions.
    \item \textbf{RMIA:} Builds a robust likelihood ratio based on LiRA.
    \item \textbf{IAM:} Uses a model interpolation strategy to localize model information for auditing.
    \item \textbf{RULI:} Explicitly models the unlearning process for auditing.
\end{itemize}

\textbf{MIA Settings.} For all MIA baselines, we use the online white-box setting, where the attacker accesses the pre-unlearning model, the post-unlearning model, and all shadow models. This gives the baselines the strongest capability under our auditing setup. For LiRA-type methods, we use the full logit vector and learn the threshold on an independent calibration set with 10,000 member and 10,000 non-member samples. The learned threshold is then applied to the unlearning data, with calibration and evaluation strictly separated. For MIA methods with class-level decisions, we report the better result between class-level and full-sample evaluation as a strong upper bound.

\textbf{Metrics.} MIA methods output sample-level membership decisions, while SMI methods estimate a distribution-level mixture ratio. For MIA, we report FNR, FPR, and the proportion of unlearning samples classified as non-members as the aligned $\rho^*$. For SMI, we report the estimated $\rho^*$.

\subsection{The Performance of SMI}
\textbf{Unlearning Task Settings.} In this subsection, we use full retraining as the optimal unlearned model auditing target, and set the unlearning data to a randomly selected 5\% or 10\% subset of each dataset. Practical unlearning algorithms are discussed in the following subsection.

\begin{table}[t]
\centering
\caption{Auditing performance of SMI and MIA-based methods on complete unlearning. The Dataset column lists the model, task, unlearning percentage, and number of shadow models. Experimental results for some datasets, with the \textbf{best} and \underline{second best} results are highlighted. The complete results can be found in Appendix~\ref{A:zhubiao}.}
\label{tab:Results12}
\resizebox{\textwidth}{!}{%
\begin{tabular}{@{}llGGccccccc@{}}
\toprule
Dataset & Metric & SMI & SMI-M & RULI & IAM & RMIA & EMIA & LiRA & Unleak & Ramdom \\
\midrule
\multirow{3}{*}{Resnet18-Cifar10-5\%-64} & FNR & -- & -- & 13.28\% & 12.72\% & 49.68\% & 70.82\% & 13.80\% & 41.24\% & 49.52\% \\
 & FPR & -- & -- & 0.00\% & 13.08\% & 50.04\% & 70.72\% & 0.00\% & 43.48\% & 50.88\% \\
 & $\rho^*$ & \underline{88.89\%} & \textbf{92.44\%} & 87.94\% & 88.67\% & 66.68\% & 39.10\% & 87.84\% & 62.58\% & 50.06\% \\
\midrule
\multirow{3}{*}{Resnet18-Cifar100-5\%-16} & FNR & -- & -- & 15.60\% & 17.76\% & 87.80\% & 24.72\% & 10.47\% & 81.04\% & 48.72\% \\
 & FPR & -- & -- & 0.00\% & 17.96\% & 87.96\% & 22.24\% & 0.00\% & 82.88\% & 49.52\% \\
 & $\rho^*$ & 79.64\% & \textbf{89.84\%} & 83.00\% & \underline{89.36\%} & 67.19\% & 84.64\% & 86.64\% & 28.52\% & 50.88\% \\
\midrule
\multirow{3}{*}{Resnet50-Cifar100-5\%-32} & FNR & -- & -- & 36.76\% & 86.72\% & 91.92\% & 66.52\% & 37.41\% & 93.08\% & 47.68\% \\
 & FPR & -- & -- & 0.00\% & 63.83\% & 56.00\% & 57.56\% & 0.00\% & 47.69\% & 50.15\% \\
 & $\rho^*$ & 81.22\% & \textbf{92.64\%} & 73.08\% & 63.56\% & 58.24\% & 67.84\% & \underline{87.04\%} & 10.40\% & 54.68\% \\
\midrule
\multirow{3}{*}{Resnet50-Cifar100-10\%-32} & FNR & -- & -- & 26.90\% & 87.73\% & 91.50\% & 70.72\% & 38.29\% & 93.56\% & 49.71\% \\
 & FPR & -- & -- & 0.00\% & 61.69\% & 53.21\% & 57.81\% & 0.00\% & 49.77\% & 42.27\% \\
 & $\rho^*$ & 83.46\% & \textbf{93.95\%} & 74.44\% & 64.98\% & 60.74\% & 69.94\% & \underline{87.38\%} & 4.68\% & 54.14\% \\
\midrule
\multirow{3}{*}{Resnet50-Cinic10-10\%-16} & FNR & -- & -- & 11.36\% & 29.92\% & 11.88\% & 22.08\% & 11.18\% & 19.46\% & 49.82\% \\
 & FPR & -- & -- & 0.00\% & 30.96\% & 11.28\% & 21.68\% & 0.00\% & 14.04\% & 50.23\% \\
 & $\rho^*$ & \underline{90.89\%} & \textbf{92.27\%} & 89.88\% & 85.35\% & 87.64\% & 89.66\% & 90.47\% & 78.02\% & 49.76\% \\
\midrule
\multirow{3}{*}{VIT-Cinic10-10\%-16} & FNR & -- & -- & 15.92\% & 57.40\% & 85.04\% & 65.36\% & 18.76\% & 43.24\% & 49.97\% \\
 & FPR & -- & -- & 0.00\% & 53.37\% & 86.12\% & 65.24\% & 0.00\% & 44.76\% & 51.93\% \\
 & $\rho^*$ & \underline{91.44\%} & \textbf{93.62\%} & 89.84\% & 41.01\% & 15.33\% & 36.66\% & 83.51\% & 55.71\% & 51.04\% \\
\bottomrule
\end{tabular}
}
\end{table}
\textbf{Results.}
Table~\ref{tab:Results12} reports the results across datasets for both MIA- and SMI-based auditing. Overall, SMI-M achieves the best estimation accuracy. Among non-SMI methods, RuLI and LiRA perform best in most settings: RuLI explicitly models the unlearning data, while LiRA uses a simpler and more stable auditing pipeline.Better auditing performance is usually associated with lower FPR; for example, RuLI and LiRA consistently achieve an FPR of 0. However, different metrics are not strictly aligned. In many settings, LiRA has a higher FNR than RuLI, whereas RuLI gives a more reliable $\rho^*$. This supports our corollary: even a well-trained MIA may still suffer from an additional auditing error in unlearned model auditing.
\subsection{The Computational Cost of SMI}
\textbf{Computational Cost.}
Table~\ref{table2} reports the computational overhead of different MIA methods and SMI variants for a single auditing run on CIFAR-100. For SMI methods, the runtime covers auditing over all classes. Methods without shadow models generally have low cost, although SMI-M is slower due to its kernel embedding computation.This overhead is still small compared with training shadow models. For example, training one LiRA shadow model with a ResNet-18 backbone takes enough time to run SMI-M hundreds of times, while shadow-model-based MIA usually requires multiple shadow models. For SMI, the main cost comes from bootstrapping, and the runtimes in Table~\ref{table2} are reported with 200 bootstrap iterations. Non-shadow methods require one inference pass, except Unleak, which requires two. Shadow-model methods require one pass for each shadow model plus one additional pass.
\begin{table}[htbp]
\centering
\captionof{table}{Computational cost of MIA methods. A dash (–) indicates that the corresponding method is not used. The compute time reported for Train shadow model refers to the cost of training a single shadow model.}
\resizebox{\linewidth}{!}{%
\begin{tabular}{lccc}
\toprule
Method & Compute times & Shadow model & Bootstrap \\
\midrule
Inference on one model & $4.47 \times 10^4$ms & - & - \\
Logit Metric & 217.25ms & - & - \\
Unleak & 667.09ms & - & - \\
SMI & 930.71ms & - & 200 \\
SMI-M & 4301.66ms & - & 200 \\
\midrule
Train one shadow model & $2.46 \times 10^6$ms & - & - \\
LiRA & 1464.51ms & 8 & - \\
EMIA & 1625.26ms & 8 & - \\
RMIA & 2475.42ms & 8 & - \\
IAM & 2944.35ms & 8 & - \\
RULI & 5672.83ms & 8 & - \\
\bottomrule
\end{tabular}
}
\label{table2}
\end{table}

\subsection{The Small-Sample Reliability of SMI}
\begin{figure}
\vspace{-1.0em}
\centering
    \includegraphics[width=0.6\linewidth]{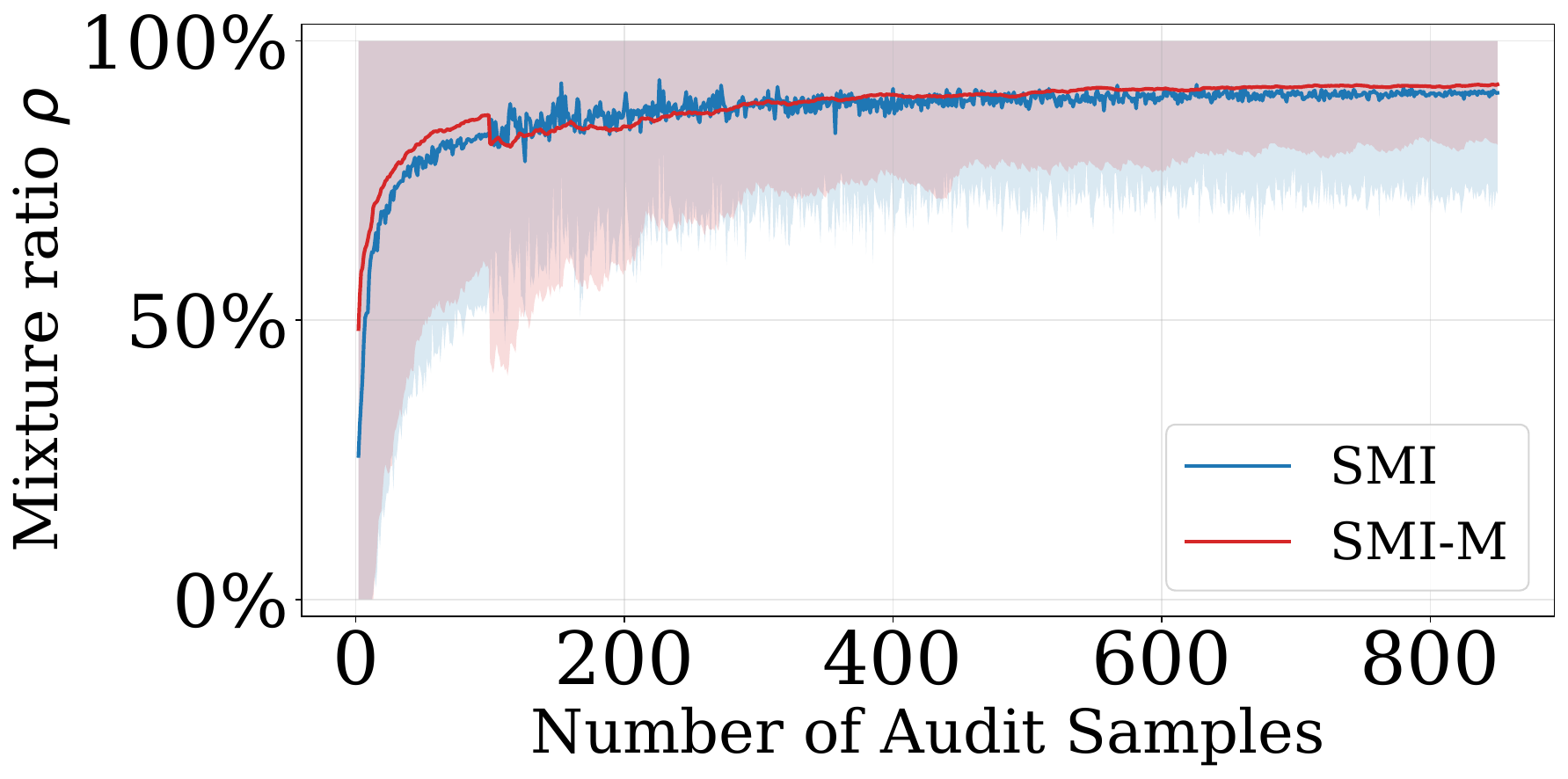}
    \caption{CINIC-ResNet50}
    \label{fig:cinic-resnet50-b}
\caption{Effect of audit sample size on SMI and SMI-M performance. Solid lines show the mean, and dashed lines show the reference range.}
\label{fig:learnicurves}
\vspace{-1.0em}
\end{figure}
\textbf{The Small-Sample Challenge of SMI.} Since SMI relies on distributional mixture estimation, its error increases with smaller sample sizes. We therefore study its sample requirement on CINIC-10 by varying the audit size from 1 to 800. For each size, we run 1,000 resampling audits when the size is below 200 and 500 audits otherwise. As shown in Figure~\ref{fig:learnicurves}, SMI converges around 200 samples, or about 20 samples per class for a 10-class dataset.
\begin{wrapfigure}{r}{0.49\textwidth}
\vspace{-1.0em}
\centering
\begin{subfigure}[t]{0.48\linewidth}
    \centering
    \includegraphics[width=\linewidth]{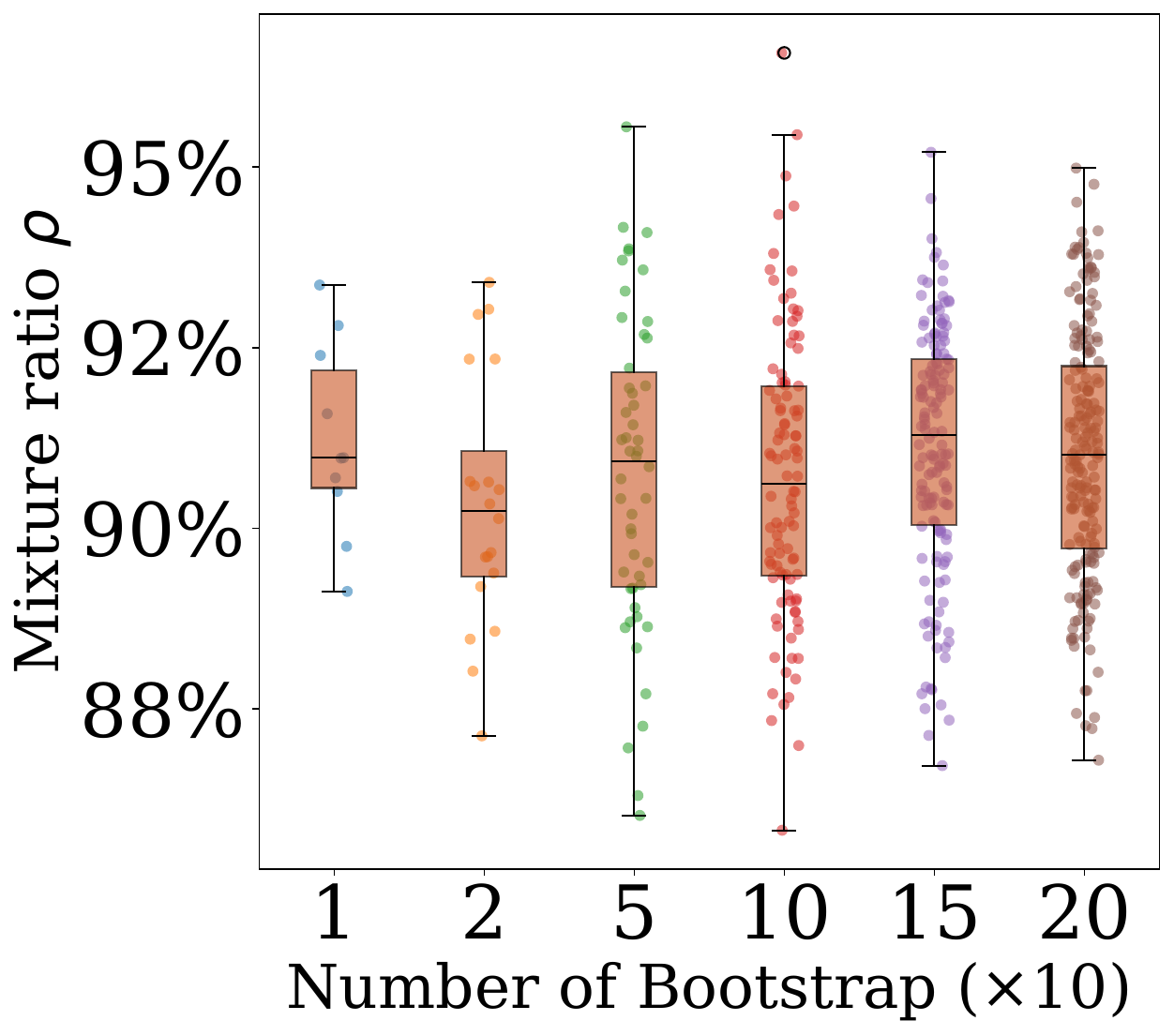}
    \caption{SMI}
    \label{fig:cinic-resnet50-a}
\end{subfigure}
\hfill
\begin{subfigure}[t]{0.48\linewidth}
    \centering
    \includegraphics[width=\linewidth]{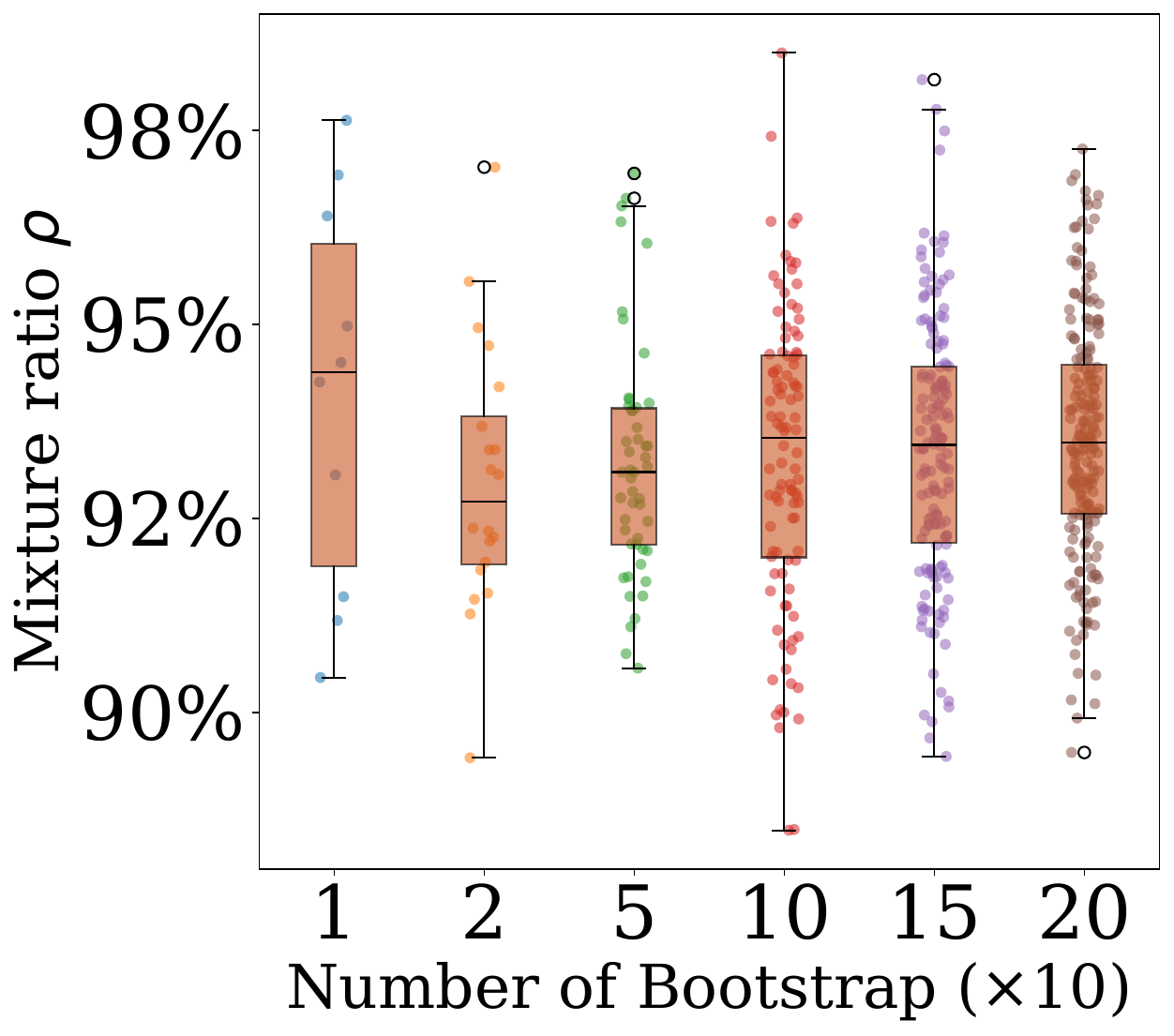}
    \caption{SMI-M}
    \label{fig:cinic-resnet50-b}
\end{subfigure}
\caption{Performance fluctuation of SMI and SMI-M with different numbers of resampling iterations on 10\% randomly unlearned data.}
\label{fig:learning-444}
\vspace{-1.0em}
\end{wrapfigure}

\textbf{Bootstrap Overhead.} Figure~\ref{fig:learning-444} illustrates the auditing performance of SMI and SMI-M under different numbers of bootstrap groups. As a white-box auditing method, SMI has already demonstrated strong and reliable performance advantages. Across varying numbers of bootstrap groups, SMI does not exhibit significant fluctuations in either the median performance or the reference ranges, indicating that the bootstrap procedure is well controlled and does not need to be excessively increased to improve SMI’s performance.
\subsection{The Robustness of SMI}
\textbf{Hyperparameter.} Since SMI-M adopts a kernel embedding strategy, it involves several hyperparameters that need to be discussed. The hyperparameters of MMD generally include the choice of kernel function and the bandwidth selection strategy. In this subsection, we consider the RBF kernel, Laplacian kernel, linear kernel, polynomial kernel, and cosine kernel, as well as different RBF-based bandwidth selection strategies. The experimental results are shown in the table~\ref{tab:kernel_comparison_cifar100}. Overall, the standard median bandwidth with the RBF kernel achieves the best MMD hyperparameter setting. The bandwidth is computed as the median distance over 10,000 sampled points from the member, non-member, and unlearning data. If the number of unlearning samples is small, the bandwidth can be estimated using only the member and non-member data.
\begin{table}[t]
\centering
\caption{Relationship between different kernels and the performance of SMI-M. Values in parentheses denote the multiplier applied to the median bandwidth. }
\label{tab:kernel_comparison_cifar100}
\resizebox{\textwidth}{!}{%
\begin{tabular}{lccccccccc}
\toprule
Data & RBF & RBF(0.25X) & RBF(0.5X) & RBF(2X) & RBF(4X) & Laplacian & Linear & Polynomial & Consine \\
\midrule
Cifar100-5\%  & 89.84\% & 24.39\% & 77.70\% & 86.88\% & 86.72\% & 80.05\% & 86.65\% & 82.56\% & 83.98\% \\
Cifar100-10\% & 91.84\% & 26.03\% & 79.22\% & 89.88\% & 89.90\% & 82.41\% & 89.88\% & 85.63\% & 88.16\% \\
\bottomrule
\end{tabular}
}
\end{table}
\subsection{The Practicality of SMI}
\textbf{Practical Tasks.}
Beyond retraining, practical scenarios involve more approximate unlearning tasks that are difficult to audit. These methods may exhibit low accuracy on $D_u$, but their true level of forgetting is hard to determine. In this subsection, we evaluate the unlearning performance of three approximate unlearning strategies other than retraining: FineTune, Fisher, and Scrub. Their detailed descriptions are provided in Appendix~\ref{unlearning}. The final results are shown in Table~\ref{tab:comparison_cifar100_cinic}. Among them, Scrub achieves the best unlearning performance.

\begin{table}[h]
\centering
\caption{SMI-based Detection of Unlearning Levels in Practical Unlearning Tasks}
\label{tab:comparison_cifar100_cinic}
\resizebox{\textwidth}{!}{%
\begin{tabular}{lcccccccc}
\toprule
\multirow{2}{*}{Method}
& \multicolumn{4}{c}{Cifar100}
& \multicolumn{4}{c}{Cinic} \\
\cmidrule(lr){2-5} \cmidrule(lr){6-9}
& Resnet18-5\% & Resnet18-10\% & Resnet50-5\% & Resnet50-10\%
& Resnet50-5\% & Resnet50-10\% & Vit-5\% & Vit-10\% \\
\midrule
Retrain  & 89.84\% & 91.84\% & 92.64\% & 93.95\% & 90.72\% & 92.27\% & 92.13\% & 93.62\% \\
\midrule
FineTune & 63.01\% & 66.54\% & 59.63\% & 61.79\% & 67.35\% & 68.66\% & 58.99\% & 55.42\% \\
Fisher   & 32.51\% & 34.70\% & 39.82\% & 34.86\% & 40.83\% & 44.98\% & 44.77\% & 38.93\% \\
Scrub     & 76.34\% & 75.95\% & 72.61\% & 77.58\% & 69.21\% & 63.27\% & 62.43\% & 66.16\% \\
\bottomrule
\end{tabular}
}
\end{table}

\section{Conclusion}
In this work, we propose SMI, an unlearned model auditing method based on measuring distributional distances. We revisit the implicit bounds underlying MIA attacks and MIA-based auditing, and argue that failed attacks can produce illusory forgetting. To address this issue, we propose a mixture-ratio estimation strategy for auditing model unlearning, which provides an efficient unlearned model auditing approach without requiring shadow models. Experimental results show that our MMD-based SMI-M outperforms existing MIA-based methods. It demonstrates strong capability in reliably measuring unlearning outcomes, achieving efficient computation without training shadow models, and providing reference intervals to assess the reliability of the auditing conclusions.

\section*{Acknowledgements}
We thank the following non-author contributors for their valuable contributions and insights, which greatly enriched the manuscript: Xinpeng Ling from Tongji University for suggestions on the writing; Zhanpeng Shi from Shanghai Jiao Tong University for suggestions on the writing; Sizhe Liu from The Chinese University of Hong Kong for advice on probability theory; Baolin Chen from Capital Normal University for valuable insights into high-dimensional statistics; Yihao Xiao from Shanghai University of Finance and Economics for help with interpretable AI; Kejia Zhang from Heilongjiang University for analysis from an AI safety perspective; and Zihui Song from Heilongjiang University for assistance with several figures. We sincerely thank all of them for their support.

\bibliographystyle{plain}
\bibliography{ref}

\begin{thebibliography}{10}

\bibitem{Sha2}
Li~Bai, Qingqing Ye, Xinwei Zhang, Sen Zhang, Zi~Liang, Jianliang Xu, and Haibo Hu.
\newblock Toward efficient inference attacks: Shadow model sharing via mixture-of-experts.
\newblock {\em arXiv preprint arXiv:2510.13451}, 2025.

\bibitem{Lira}
Nicholas Carlini, Steve Chien, Milad Nasr, Shuang Song, Andreas Terzis, and Florian Tramer.
\newblock Membership inference attacks from first principles.
\newblock In {\em 2022 IEEE symposium on security and privacy (SP)}, pages 1897--1914. IEEE, 2022.

\bibitem{Unleark}
Min Chen, Zhikun Zhang, Tianhao Wang, Michael Backes, Mathias Humbert, and Yang Zhang.
\newblock When machine unlearning jeopardizes privacy.
\newblock In {\em Proceedings of the 2021 ACM SIGSAC Conference on Computer and Communications Security}, CCS '21, page 896–911, New York, NY, USA, 2021. Association for Computing Machinery.

\bibitem{cinic}
Luke~N Darlow, Elliot~J Crowley, Antreas Antoniou, and Amos~J Storkey.
\newblock Cinic-10 is not imagenet or cifar-10.
\newblock {\em arXiv preprint arXiv:1810.03505}, 2018.

\bibitem{Sha1}
Yuntao Du, Jiacheng Li, Yuetian Chen, Kaiyuan Zhang, Zhizhen Yuan, Hanshen Xiao, Bruno Ribeiro, and Ninghui Li.
\newblock Cascading and proxy membership inference attacks.
\newblock {\em arXiv preprint arXiv:2507.21412}, 2025.

\bibitem{unlearningicml2}
Ali Ebrahimpour-Boroojeny, Hari Sundaram, and Varun Chandrasekaran.
\newblock Not all wrong is bad: Using adversarial examples for unlearning.
\newblock In Aarti Singh, Maryam Fazel, Daniel Hsu, Simon Lacoste-Julien, Felix Berkenkamp, Tegan Maharaj, Kiri Wagstaff, and Jerry Zhu, editors, {\em Proceedings of the 42nd International Conference on Machine Learning}, volume 267 of {\em Proceedings of Machine Learning Research}, pages 14950--14971. PMLR, 13--19 Jul 2025.

\bibitem{Jeffsafty}
Jie Fu, Yuan Hong, Zhili Chen, and Wendy~Hui Wang.
\newblock Safeguarding graph neural networks against topology inference attacks.
\newblock In {\em Proceedings of the 2025 ACM SIGSAC Conference on Computer and Communications Security}, pages 2144--2158, 2025.

\bibitem{back1}
Tianyu Gu, Kang Liu, Brendan Dolan-Gavitt, and Siddharth Garg.
\newblock Badnets: Evaluating backdooring attacks on deep neural networks.
\newblock {\em Ieee Access}, 7:47230--47244, 2019.

\bibitem{MIA2}
Ling Han, Hao Huang, Dustin Scheinost, Mary-Anne Hartley, and Mar{\'\i}a~Rodr{\'\i}guez Mart{\'\i}nez.
\newblock Unlearning information bottleneck: Machine unlearning of systematic patterns and biases.
\newblock {\em arXiv preprint arXiv:2405.14020}, 2024.

\bibitem{newunlearning}
Jamie Hayes, Ilia Shumailov, Eleni Triantafillou, Amr Khalifa, and Nicolas Papernot.
\newblock Inexact unlearning needs more careful evaluations to avoid a false sense of privacy.
\newblock In {\em 2025 IEEE Conference on Secure and Trustworthy Machine Learning (SaTML)}, pages 497--519. IEEE, 2025.

\bibitem{renset18}
Kaiming He, Xiangyu Zhang, Shaoqing Ren, and Jian Sun.
\newblock Deep residual learning for image recognition.
\newblock In {\em Proceedings of the IEEE conference on computer vision and pattern recognition}, pages 770--778, 2016.

\bibitem{lilun1}
Jean Honorio and Tommi Jaakkola.
\newblock Tight bounds for the expected risk of linear classifiers and pac-bayes finite-sample guarantees.
\newblock In {\em Artificial Intelligence and Statistics}, pages 384--392. PMLR, 2014.

\bibitem{duliang3}
Koulik Khamaru, Yash Deshpande, Tor Lattimore, Lester Mackey, and Martin~J Wainwright.
\newblock Near-optimal inference in adaptive linear regression.
\newblock {\em The Annals of Statistics}, 53(6):2329--2355, 2025.

\bibitem{RESNET50}
Brett Koonce.
\newblock Resnet 50.
\newblock In {\em Convolutional neural networks with swift for tensorflow: image recognition and dataset categorization}, pages 63--72. Springer, 2021.

\bibitem{Cifar}
Alex Krizhevsky, Geoffrey Hinton, et~al.
\newblock Learning multiple layers of features from tiny images.(2009), 2009.

\bibitem{unlearningN}
Uyen~N Le-Khac and Vinh~NX Truong.
\newblock A survey on large language models unlearning: taxonomy, evaluations, and future directions.
\newblock {\em Artificial Intelligence Review}, 58(12):399, 2025.

\bibitem{backandduikang}
Sijia Liu, Yuanshun Yao, Jinghan Jia, Stephen Casper, Nathalie Baracaldo, Peter Hase, Yuguang Yao, Chris~Yuhao Liu, Xiaojun Xu, Hang Li, et~al.
\newblock Rethinking machine unlearning for large language models.
\newblock {\em Nature Machine Intelligence}, pages 1--14, 2025.

\bibitem{backdoor}
Yang Liu, Mingyuan Fan, Cen Chen, Ximeng Liu, Zhuo Ma, Li~Wang, and Jianfeng Ma.
\newblock Backdoor defense with machine unlearning.
\newblock In {\em IEEE INFOCOM 2022-IEEE conference on computer communications}, pages 280--289. IEEE, 2022.

\bibitem{Liusafty}
Yi~Liu, Weixiang Han, Chengjun Cai, Xingliang Yuan, and Cong Wang.
\newblock Privtune: Efficient and privacy-preserving fine-tuning of large language models via device-cloud collaboration, 2026.

\bibitem{Yisafty}
Yi~Liu, Lei Xu, Xingliang Yuan, Cong Wang, and Bo~Li.
\newblock The right to be forgotten in federated learning: An efficient realization with rapid retraining.
\newblock In {\em IEEE INFOCOM 2022-IEEE conference on computer communications}, pages 1749--1758. IEEE, 2022.

\bibitem{MIAD}
Zihao Luo, Xilie Xu, Feng Liu, Yun~Sing Koh, Di~Wang, and Jingfeng Zhang.
\newblock Privacy-preserving low-rank adaptation against membership inference attacks for latent diffusion models.
\newblock In {\em Proceedings of the AAAI Conference on Artificial Intelligence}, volume~39, pages 5883--5891, 2025.

\bibitem{Sha3}
Ryoto Miyamoto, Xin Fan, Fuyuko Kido, Tsuneo Matsumoto, and Hayato Yamana.
\newblock Openlvlm-mia: A controlled benchmark revealing the limits of membership inference attacks on large vision-language models.
\newblock {\em arXiv preprint arXiv:2510.16295}, 2025.

\bibitem{newmia1}
Nima Naderloui, Shenao Yan, Binghui Wang, Jie Fu, Wendy~Hui Wang, Weiran Liu, and Yuan Hong.
\newblock Rectifying privacy and efficacy measurements in machine unlearning: A new inference attack perspective.
\newblock In {\em 34th USENIX Security Symposium (USENIX Security 25)}, pages 5545--5564, 2025.

\bibitem{Miawhite}
Milad Nasr, Reza Shokri, and Amir Houmansadr.
\newblock Comprehensive privacy analysis of deep learning: Passive and active white-box inference attacks against centralized and federated learning.
\newblock In {\em 2019 IEEE Symposium on Security and Privacy (SP)}, pages 739--753, 2019.

\bibitem{unlearningS}
Thanh~Tam Nguyen, Thanh~Trung Huynh, Zhao Ren, Phi~Le Nguyen, Alan Wee-Chung Liew, Hongzhi Yin, and Quoc Viet~Hung Nguyen.
\newblock A survey of machine unlearning.
\newblock {\em ACM Transactions on Intelligent Systems and Technology}, 16(5):1--46, 2025.

\bibitem{duikang}
Serena Nicolazzo, Antonino Nocera, et~al.
\newblock How secure is forgetting? linking machine unlearning to machine learning attacks.
\newblock {\em arXiv preprint arXiv:2503.20257}, 2025.

\bibitem{lilun2}
Yuki Ohnishi and Jean Honorio.
\newblock Novel change of measure inequalities with applications to pac-bayesian bounds and monte carlo estimation.
\newblock In {\em International conference on artificial intelligence and statistics}, pages 1711--1719. PMLR, 2021.

\bibitem{duliang2}
Harish Ramaswamy, Clayton Scott, and Ambuj Tewari.
\newblock Mixture proportion estimation via kernel embeddings of distributions.
\newblock In {\em International conference on machine learning}, pages 2052--2060. PMLR, 2016.

\bibitem{MIA}
Reza Shokri, Marco Stronati, Congzheng Song, and Vitaly Shmatikov.
\newblock Membership inference attacks against machine learning models.
\newblock In {\em 2017 IEEE symposium on security and privacy (SP)}, pages 3--18. IEEE, 2017.

\bibitem{duikang1}
Christian Szegedy, Wojciech Zaremba, Ilya Sutskever, Joan Bruna, Dumitru Erhan, Ian Goodfellow, and Rob Fergus.
\newblock Intriguing properties of neural networks.
\newblock {\em arXiv preprint arXiv:1312.6199}, 2013.

\bibitem{duliang1}
G{\'a}bor~J Sz{\'e}kely, Maria~L Rizzo, and Nail~K Bakirov.
\newblock Measuring and testing dependence by correlation of distances.
\newblock 2007.

\bibitem{IAM}
Cheng-Long Wang, Qi~Li, Zihang Xiang, Yinzhi Cao, and Di~Wang.
\newblock Towards lifecycle unlearning commitment management: Measuring sample-level unlearning completeness.
\newblock {\em arXiv preprint arXiv:2506.06112}, 2025.

\bibitem{miaCCS}
Zhiqi Wang, Chengyu Zhang, Yuetian Chen, Nathalie Baracaldo, Swanand~R Kadhe, and Lei Yu.
\newblock Membership inference attacks as privacy tools: Reliability, disparity and ensemble.
\newblock In {\em Proceedings of the 2025 ACM SIGSAC Conference on Computer and Communications Security}, pages 1724--1738, 2025.

\bibitem{unlearningicml1}
Rongzhe Wei, Mufei Li, Mohsen Ghassemi, Eleonora Kreačić, Yifan Li, Xiang Yue, Bo~Li, Vamsi~K. Potluru, Pan Li, and Eli Chien.
\newblock Underestimated privacy risks for minority populations in large language model unlearning, 2025.

\bibitem{wu2025reliable}
Chengcan Wu, Zeming Wei, Huanran Chen, Yinpeng Dong, and Meng Sun.
\newblock Reliable unlearning harmful information in llms with metamorphosis representation projection.
\newblock In {\em NeurIPS Workshop on Reliable ML from Unreliable Data}, 2025.

\bibitem{unlearningnunix}
Dayong Ye, Tianqing Zhu, Jiayang Li, Kun Gao, Bo~Liu, Leo~Yu Zhang, Wanlei Zhou, and Yang Zhang.
\newblock Data duplication: A novel multi-purpose attack paradigm in machine unlearning, 2025.

\bibitem{vit}
Li~Yuan, Yunpeng Chen, Tao Wang, Weihao Yu, Yujun Shi, Zi-Hang Jiang, Francis~EH Tay, Jiashi Feng, and Shuicheng Yan.
\newblock Tokens-to-token vit: Training vision transformers from scratch on imagenet.
\newblock In {\em Proceedings of the IEEE/CVF international conference on computer vision}, pages 558--567, 2021.

\bibitem{RMIA}
Sajjad Zarifzadeh, Philippe Liu, and Reza Shokri.
\newblock Low-cost high-power membership inference attacks.
\newblock In Ruslan Salakhutdinov, Zico Kolter, Katherine Heller, Adrian Weller, Nuria Oliver, Jonathan Scarlett, and Felix Berkenkamp, editors, {\em Proceedings of the 41st International Conference on Machine Learning}, volume 235 of {\em Proceedings of Machine Learning Research}, pages 58244--58282. PMLR, 21--27 Jul 2024.

\bibitem{New1}
Chenhao Zhang, Muxing Li, Feng Liu, Weitong Chen, and Miao Xu.
\newblock Unlearning evaluation through subset statistical independence.
\newblock {\em arXiv preprint arXiv:2603.00587}, 2026.

\bibitem{nounealrning}
Qingjie Zhang, Haoting Qian, Zhicong Huang, Cheng Hong, Minlie Huang, Ke~Xu, Chao Zhang, and Han Qiu.
\newblock Understanding the dilemma of unlearning for large language models.
\newblock {\em arXiv preprint arXiv:2509.24675}, 2025.

\bibitem{nounlearning1}
Junhao Zheng, Xidi Cai, Shengjie Qiu, and Qianli Ma.
\newblock Spurious forgetting in continual learning of language models.
\newblock {\em arXiv preprint arXiv:2501.13453}, 2025.

\end{thebibliography}

\clearpage
\appendix

\section*{Overview}
For a detailed exposition of SMI, this appendix is organized into five parts: Related Work, Disscussion and Limitations, Theory Proofs,  Theoretical Elaboration on SMI-M and MMD and The Complete Experimental.
\section{Related Work}
\subsection{Machine unlearning}
Machine unlearning aims to selectively remove the influence of the data to be unlearned, $D_u$, from a trained model. Let the initial model $w_0$ be trained on the member training data $D_m$, where $D_u$ denotes the unlearning set to be removed and $D_n$ denotes non-member data that have never participated in training~\cite{unlearningicml2}. The goal of machine unlearning is to eliminate the influence of $D_u$ while preserving the model utility on $D_m$ as much as possible, so that the unlearned model $w_u$ becomes close to the retrained model $w_r$ obtained by training only on $D_m\setminus D_u$~\cite{unlearningN,unlearningS,Yisafty}.
\subsection{Auditing Strategies for Machine Unlearning}
A common privacy-oriented evaluation strategy for machine unlearning is to regard a model trained from scratch without the unlearning data as the optimal unlearned model. This view is based on the premise that, if an unlearning algorithm truly removes the influence of the unlearning data, the resulting model should behave similarly to this retrained model. Following this idea, existing studies usually evaluate unlearning by comparing the unlearned model and the retrained model in terms of parameters, output distributions, prediction confidence, or task performance~\cite{New1,unlearningS}.

Another intuitive privacy-oriented evaluation method is MIA. In the context of machine unlearning, Chen et al. first showed that unlearning itself may introduce new privacy risks: by comparing the outputs of the original model and the unlearned model, an attacker can infer whether a target sample was included in the original training set but removed from the unlearned model~\cite{Unleark}. Beyond this unlearning-specific leakage, Carlini et al. formulated membership inference from first principles and proposed a likelihood-ratio attack that provides a stronger and more systematic way to detect membership traces~\cite{Lira}. Zarifzadeh et al. further improved the practicality of MIA by developing a low-cost, high-power attack that remains effective even with limited reference models~\cite{RMIA}. More recently, Wang et al. argued that binary membership inference alone is insufficient for measuring approximate unlearning, and proposed sample-level unlearning completeness to characterize how completely each individual sample has been forgotten~\cite{IAM}. Naderloui et al. further revisited privacy and efficacy measurements in machine unlearning from a new inference attack perspective, showing that unlearning evaluation requires attacks tailored to the post-unlearning setting rather than relying only on coarse aggregate metrics~\cite{newmia1}. Together, these studies motivate MIA-based evaluation as an interpretable privacy audit: if the unlearning data have been effectively removed, they should no longer exhibit membership traces that allow an attacker to distinguish them from non-member samples.

\section{Discussion and Limitations}
To further elucidate the workings of SMI, we address potential queries regarding SMI in advance to facilitate a better understanding of its underlying philosophy:

\textbf{Question 1:} The contribution of SMI appears to be merely replacing the binary classifier with a statistical metric discriminator; this contribution seems incremental.

\textbf{Answer:} From a technical standpoint, the contribution of SMI differs significantly from traditional MIA methods, which focus on constructing a $g(x;w)$ with superior performance. The core contribution of SMI is the proposal of a novel perspective: When the attack on a single sample becomes unreliable as an evaluation metric, directly measuring statistical indicators for the entire audit population is a more reliable approach. In reality, indistinguishable samples always exist; however, when viewed through the lens of statistical measurement, they are naturally categorized within the distribution of member samples. Conversely, if these indistinguishable samples are involved in the training of a binary classifier as non-member data, they disrupt the classifier's training process.

\textbf{Question 2:} Why not consider using Wasserstein methods to improve mixture proportion estimation, given their natural alignment with the Monge problem?

\textbf{Answer:} This is an incisive observation. We initially adopted SMI-W as our working strategy, hoping to formulate it as an improvement over SMI-M. However, our experiments showed that its performance was inferior to the simple estimation strategy used by SMI.

The formalization of the Monge problem helped us understand why the Wasserstein approach failed in this setting. Wasserstein distance measures the optimal discrepancy between two distributions. In our context, it can be viewed as a metric for an optimal forgetting strategy defined over the parameter space, or equivalently, as measuring an optimal transformation process in a modified parameter or feature space. However, neither retraining nor fine-tuning a model fundamentally satisfies the geometric assumptions underlying the Wasserstein formulation: the actual update trajectory does not align with the optimal transport geodesic.

Put more simply, Wasserstein distance assumes that there exists a shortest geodesic between two distributions and measures the length of that path. Yet the forgetting behavior we truly need to quantify does not lie on this path; from the perspective of the Wasserstein metric, all feasible forgetting strategies we can conceive of are too crude. This also motivates our future work: in high-dimensional spaces, Wasserstein metrics and gradient flows may not necessarily be aligned, and such statistical distances must therefore be used with caution.

\textbf{Question 3:} What advantages does SMI offer over MIA as an auditing strategy in practical audit settings?

\textbf{Answer:} We believe the most important advantage is that, during the auditing process, one does not need to worry about the effectiveness of the auditing tool itself. According to our experimental results, the greatest strength of SMI-M is its exceptional stability: it does not exhibit large performance fluctuations during auditing.

By contrast, a major limitation of MIA-based methods is that, before conducting the audit, one must assess whether the designed MIA discriminator has been sufficiently trained. In practice, however, it is difficult to determine whether MIA training is truly sufficient. Our experiments further show that, under the same FPR, a lower FNR does not necessarily lead to more accurate MIA decisions.

The SMI family provides a set of accurate and efficient auditing metrics whose error scales almost solely with the size of the unlearning set. Before applying SMI to a real forgetting audit, one can directly evaluate its task-specific performance on a test set.

\textbf{Question 4:} I noticed that the PAC-Bayes theory you use appears to differ from the commonly used KL-divergence-based formulation. Could you elaborate on this difference?

\textbf{Answer:} We recognize that this issue is of particular interest to researchers working on the theoretical side, and we would like to clarify why we adopt the $\chi^2$ divergence.

At the beginning, we attempted to use the most classical PAC-Bayes theorem based on KL divergence. However, while developing the proof, our literature review led us to a PAC-Bayes strategy based on $\chi^2$ divergence. We found that this formulation is more useful for decomposing the statistical error and deriving the resulting auditing error.

More interestingly, it is also better aligned with the methodology underlying our experiments. In essence, our approach solves the distributional mixture problem by exploiting higher-order statistical moments. This inevitably requires a statistical measure that can capture higher-order moment information. In this setting, $\chi^2$ divergence is substantially more effective than KL divergence at capturing second-order moment behavior.

\textbf{Limitation:} As the saying goes, there is no free lunch. Since SMI does not use shadow models as auxiliary information, it must inevitably cope with the noise introduced by small sample sizes. As a statistical strategy, SMI relies on having a sufficient number of samples. Its auditing performance on a single sample is weaker than that of MIA methods, which can exploit both the sample itself and information from shadow models. Empirically, our results suggest that at least 20 samples are needed to conduct an effective audit for a given class.

\subsection{Discussion on the Relaxation of the Monge Problem}\label{diss}
We note that some readers may have questions about the relaxed formulation in Eq.~\eqref{mengri} derived from the Monge problem, and may question whether it is valid. We clarify this issue in this section.
\begin{equation}
    \mathcal F_u
    \approx
    (1-\rho)\mathcal F_m+\rho\mathcal F_n,
    \qquad \rho\in[0,1]. \label{mengri}
\end{equation}
First, this condition does not always hold. We construct a counterexample: when a malicious attacker uses adversarial example generation to forge the unlearning process, we find that the estimated $\rho^* \gg 10$. This demonstrates that SMI-based unlearned model auditing can quickly fail under adversarial perturbations. Nevertheless, the abnormal value of $\rho^*$ can still indicate that an anomalous result has occurred in the current unlearning process.

In addition, the objective loss obtained when solving the optimization problem can serve as a classical evaluation metric. After normalizing the input features, we find that the objective value is on the order of $10^{-1}$, which indirectly supports that the mixture-proportion assumption is likely to hold in practical scenarios.

Finally, we would like to provide an insight from the perspective of the statistical estimation process. In high-dimensional statistical metrics, we often observe that part of the computation of the Wasserstein distance can be related to mixture-ratio estimation. From this viewpoint, SMI can be regarded as a Wasserstein-like high-dimensional statistical measurement strategy. Moreover, for neural network output features with relatively low signal-to-noise ratios, capturing higher-order moments may not be particularly effective. Therefore, we are confident in this heuristic conclusion.

Here, we must clarify that if one aims to obtain a strict convexity result, additional complex assumptions would be required for the proof. However, such a proof would provide limited practical guidance, and its significance would be much weaker than our decomposition of the auditing error. In future work, we will continue along the direction of AI interpretability and attempt to analyze the behavior of neural networks in feature space, although such conclusions are not easy to obtain.
\section{Theory Proofs}\label{theory}
Regarding the detailed exposition of SMI, this appendix is organized into three sections: related discussions on the SMI framework, theoretical proofs presented in the SMI paper, and further elaboration on the MMD distance.

We extend our gratitude to the readers for reaching this point. The core philosophy of SMI stems from PAC-Bayes theory, long-tailed distribution theory, and unsupervised learning theory. In the proof and characterization of our theorems, we adopt many ideas from these works. First and foremost, we thank every reader interested in SMI theory. We look forward to building the Unlearning Audit community with you.

Specifically, the empirical risk of MIA attacks is inspired by PAC theory and MIA prior theory, leading us to construct a PAC error form for MIA. As shown in Theorem~\ref{T2.1}, this constitutes one of the core contributions of our paper. Considering that this contribution is difficult to describe simply in the main text, we discuss in detail in this appendix its impact on our work and its guiding significance for MIA auditing tools.

Learning theory investigates the gap between training error and generalization error, providing performance guarantees for a well-trained learner on new data. A model with strong generalization capability will have a theoretical generalization error that is as small as possible.

Consider a learning algorithm setting where there exists a sample set $S = \{z_1, z_2, \cdots, z_m\}$ containing $m$ samples. All these samples are drawn from the same \textbf{unknown distribution} $\mathcal{D}$, which is a probability measure on a measurable space $\mathcal{Z}$. For a supervised learning problem, this measurable space can be decomposed into a feature part and a label part, i.e., $\mathcal{Z} = \mathcal{X} \times \mathcal{Y}$, where $\mathcal{X}$ is the feature space and $\mathcal{Y}$ is the label space. If the problem is a binary classification problem, then $\mathcal{Y} = \{0, 1\}$; if it is a regression problem, then $\mathcal{Y} = \mathbb{R}$. Given a measurable hypothesis space $\mathcal{H}$ and a loss function $l: \mathcal{H} \times \mathcal{Z} \to \mathbb{R}$, the general learning objective is to find the best hypothesis $h \in \mathcal{H}$ that minimizes the \textbf{true risk} (also known as the \textbf{expected risk}):
\begin{equation}
    R_D(h) = \mathbb{E}_{z \sim \mathcal{D}} l(h, z)
\end{equation}

For classification problems, $\mathcal{H}$ can be a class of classifiers. If $\mathcal{H}$ is a set of parameterizable classifier models, then $h$ can be a weight vector. Let $\mathcal{M}(\mathcal{H})$ denote a \textbf{probability measure space} over $\mathcal{H}$. To incorporate model uncertainty, we consider using \textbf{stochastic inference} instead of \textbf{deterministic inference}. Therefore, the goal of the learning task is to provide a \textbf{posterior distribution} $\mathcal{Q} \in \mathcal{M}(\mathcal{H})$ such that the following \textbf{expected risk} is minimized:
\begin{equation}
    R_D(\mathcal{Q}) = \mathbb{E}_{h \sim \mathcal{Q}}(R_D(h)) = \mathbb{E}_{h \sim \mathcal{Q}} \mathbb{E}_{z \sim \mathcal{D}} l(h, z)
\end{equation}

Similarly, by substituting the data distribution with the dataset, the \textbf{empirical risk} is defined by the following equation:
\begin{equation}
    R_S(\mathcal{Q}) = \frac{1}{m} \sum_{i=1}^m \mathbb{E}_{h \sim \mathcal{Q}} l(h, z_i)
\end{equation}

It is important to note that since the true distribution of data is unknowable, the expected risk $R_D(\mathcal{Q})$ is generally difficult to calculate directly, while the empirical error serves as an \textbf{unbiased surrogate}. The PAC-Bayes framework provides inequalities relating expected risk and empirical risk. Define the KL divergence from $\mathcal{Q}$ to $\mathcal{P}$ as:
\begin{equation}
    KL(\mathcal{Q} \| \mathcal{P}) = \mathbb{E}_{h \sim \mathcal{Q}} \log \frac{\mathcal{Q}(h)}{\mathcal{P}(h)}
\end{equation}

\begin{remark}
In fact, the expected risk tends to be any Loss you select; it does not have a fixed formulation. However, in PAC theory, KL is generally adopted to reduce the difficulty of theoretical analysis, though we do not strictly follow this convention.
\end{remark}

Further considering that some errors in the generalization process cannot be effectively estimated, a statistical complexity term is generally used to characterize the corresponding generalization error. A specific example is given in Lemma~\ref{byaes}.

\begin{lemma} \label{byaes}
Fix $\lambda > \frac{1}{2}$, and assume the loss function takes values within a range of length $L$. For any $\delta > 0, m \in \mathbb{N}, \mathcal{D}, \mathcal{P} \in \mathcal{M}$, then with probability at least $1-\delta$, the following inequality holds for $\mathcal{Q} \in \mathcal{M}$:
\begin{equation}
     R_D(\mathcal{Q}) \leq \frac{1}{1-\frac{1}{2\lambda}}\left(R_S(\mathcal{Q}) + \frac{\lambda L}{m}\left(KL(\mathcal{Q}\| \mathcal{P})+\log\frac{1}{\delta}\right)\right)
\end{equation}
\end{lemma}

The result of Lemma~\ref{byaes} provides an upper bound on the true risk, summarized by the empirical risk and a complexity term. An important fact is that empirical risk is generally hard to improve, while decomposing the complexity term helps us understand issues arising during learning and further provides ideas for improvement. In auditing tasks, the non-equilibrium distribution of $\mathcal{P}$ is a crucial process guiding our improvement of MIA. So, how can we improve upon KL divergence? We first define some important statistical metrics.

In statistical learning theory and information theory, metrics measuring the difference between two probability distributions are usually formalized as statistical distances. Among them, $f$-Divergence is a widely used family of metrics defined by a convex function representing the difference between distributions.

\begin{definition}[$f$-Divergence]
Let $\mathcal{P}$ and $\mathcal{Q}$ be two probability distributions defined on a measurable space $\mathcal{X}$, with probability density functions $p(x)$ and $q(x)$, respectively. Given a convex function $f: \mathbb{R} \to \mathbb{R}$ satisfying $f(1)=0$, the $f$-divergence from $\mathcal{P}$ to $\mathcal{Q}$ is defined as:
\begin{equation}
    D_f(\mathcal{P}\|\mathcal{Q}) = \int_{\mathcal{X}} f\left(\frac{p(x)}{q(x)}\right) q(x) \, dx
\end{equation}
Alternatively, utilizing the Radon-Nikodym derivative, it is expressed as:
\begin{equation}
    D_f(\mathcal{P}\|\mathcal{Q}) = \int_{\mathcal{X}} f\left(\frac{d\mathcal{P}}{d\mathcal{Q}}\right) \, d\mathcal{Q}
\end{equation}
\end{definition}

Different selections of the function $f(x)$ correspond to different statistical distance metrics. The following lists several common $f$-divergences and their corresponding generating functions $f(x)$:

\begin{itemize}
    \item \textbf{KL Divergence (Kullback-Leibler Divergence):}
    \begin{equation}
        f(x) = x \ln x
    \end{equation}

    \item \textbf{Reverse KL Divergence:}
    \begin{equation}
        f(x) = -\ln x
    \end{equation}

    \item \textbf{Hellinger Distance:}
    \begin{equation}
        f(x) = (\sqrt{x} - 1)^2 \quad \text{or} \quad f(x) = 2(1 - \sqrt{x})
    \end{equation}

    \item \textbf{Total Variation Distance:}
    \begin{equation}
        f(x) = \frac{1}{2}|x - 1|
    \end{equation}

    \item \textbf{Pearson $\chi^2$-Divergence:}
    Usually defined as $f(x) = (x-1)^2$, but may also appear in different literature in the following equivalent forms (differing only by constant or linear terms):
    \begin{equation}
        f(x) = (x - 1)^2, \quad f(x) = x^2 - 1, \quad \text{or} \quad f(x) = x^2 - x
    \end{equation}

    \item \textbf{Reverse Pearson $\chi^2$-Divergence:}
    \begin{equation}
        f(x) = \frac{1}{x} - 1 \quad \text{or} \quad f(x) = \frac{1}{x} - x
    \end{equation}

    \item \textbf{Jensen-Shannon Divergence (JS Divergence):}
    \begin{equation}
        f(x) = \frac{1}{2}\left[ (x+1) \ln\left(\frac{2}{x+1}\right) + x \ln x \right]
    \end{equation}

    \item \textbf{$L_1$ Norm:}
    \begin{equation}
        f(x) = |x - 1|
    \end{equation}
\end{itemize}

Having introduced the above work, we can now proceed to the proof of Theorem~\ref{2.1}.

\subsection{Proof of Theorem 2.1} \label{T2.1}
\begin{theorem}[PAC-Bayesian Generalization Bound with $\chi^2$ Divergence]
Let $\mathcal{H}$ be a hypothesis class and $\ell: \mathcal{H} \times \mathcal{Z} \to [0,1]$ a loss function. Let $D$ be a data distribution and $S = (z_1, \dots, z_m) \sim D^m$ an i.i.d. sample. For any prior distribution $\mathcal{P}$ over $\mathcal{H}$ and any posterior distribution $\mathcal{Q}$ such that $\mathcal{Q} \ll \mathcal{P}$, with probability at least $1 - \delta$ over the draw of $S$, it holds that
\begin{equation}
    R_D(\mathcal{Q}) \leq R_S(\mathcal{Q}) + \sqrt{ \frac{2}{m\delta} \left( \chi^2(\mathcal{Q} \| \mathcal{P}) + 1 \right)  },
\end{equation}
where $R_D(\mathcal{Q}) = \mathbb{E}_{h \sim \mathcal{Q}}[R_D(h)]$, $R_S(\mathcal{Q}) = \mathbb{E}_{h \sim \mathcal{Q}}[R_S(h)]$, and the $\chi^2$-divergence is defined as
\begin{equation}
    \chi^2(\mathcal{Q} \| \mathcal{P}) = \int_{\mathcal{H}} \left( \frac{d\mathcal{Q}(h)}{d\mathcal{P}(h)} - 1 \right)^2 d\mathcal{P}(h).
\end{equation}
\end{theorem}

\begin{proof}
Let $S=(z_1,\ldots,z_m)\sim D^m$ be a training sample drawn independently from the distribution $D$. For any hypothesis $h\in\mathcal H$, define its true risk and empirical risk as
\begin{equation}
R_D(h)=\mathbb E_{z\sim D}[\ell(h,z)],
\qquad
R_S(h)=\frac{1}{m}\sum_{i=1}^m \ell(h,z_i).
\end{equation}
We further define the generalization gap of a single hypothesis as
\begin{equation}
\Delta_S(h)=R_D(h)-R_S(h).
\end{equation}
Since
\begin{equation}
\mathbb E_{S\sim D^m}[R_S(h)]
=
R_D(h),
\end{equation}
we have, for any fixed $h\in\mathcal H$,
\begin{equation}
\mathbb E_{S\sim D^m}[\Delta_S(h)]=0.
\end{equation}

For a posterior distribution $\mathcal Q$, the corresponding Gibbs true risk and empirical risk are given by
\begin{equation}
R_D(\mathcal Q)
=
\mathbb E_{h\sim\mathcal Q}[R_D(h)],
\qquad
R_S(\mathcal Q)
=
\mathbb E_{h\sim\mathcal Q}[R_S(h)].
\end{equation}
Therefore,
\begin{equation}
R_D(\mathcal Q)-R_S(\mathcal Q)
=
\mathbb E_{h\sim\mathcal Q}[\Delta_S(h)].
\end{equation}

Assume that $\mathcal Q\ll \mathcal P$, and let
\begin{equation}
r(h)=\frac{d\mathcal Q}{d\mathcal P}(h)
\end{equation}
denote the Radon--Nikodym derivative of $\mathcal Q$ with respect to $\mathcal P$. Then,
\begin{equation}
\mathbb E_{h\sim\mathcal Q}[\Delta_S(h)]
=
\int_{\mathcal H}\Delta_S(h)r(h)\,d\mathcal P(h)
=
\mathbb E_{h\sim\mathcal P}[r(h)\Delta_S(h)].
\end{equation}
By the Cauchy--Schwarz inequality, we obtain
\begin{equation}
\mathbb E_{h\sim\mathcal P}[r(h)\Delta_S(h)]
\le
\left|
\mathbb E_{h\sim\mathcal P}[r(h)\Delta_S(h)]
\right|
\le
\sqrt{\mathbb E_{h\sim\mathcal P}[r(h)^2]}
\sqrt{\mathbb E_{h\sim\mathcal P}[\Delta_S(h)^2]}.
\end{equation}
By the definition of the $\chi^2$-divergence,
\begin{equation}
\chi^2(\mathcal Q\|\mathcal P)
=
\int_{\mathcal H}(r(h)-1)^2\,d\mathcal P(h)
=
\mathbb E_{h\sim\mathcal P}[r(h)^2]-1,
\end{equation}
where we used
\begin{equation}
\mathbb E_{h\sim\mathcal P}[r(h)]=1.
\end{equation}
Hence,
\begin{equation}
\mathbb E_{h\sim\mathcal P}[r(h)^2]
=
\chi^2(\mathcal Q\|\mathcal P)+1.
\end{equation}
Let
\begin{equation}
Z_S
=
\mathbb E_{h\sim\mathcal P}[\Delta_S(h)^2].
\end{equation}
Then,
\begin{equation}
R_D(\mathcal Q)-R_S(\mathcal Q)
\le
\sqrt{\chi^2(\mathcal Q\|\mathcal P)+1}
\sqrt{Z_S}.
\end{equation}

It remains to derive a high-probability upper bound for $Z_S$. Since the prior $\mathcal P$ is independent of the sample $S$, by Fubini's theorem we have
\begin{equation}
\mathbb E_{S\sim D^m}[Z_S]
=
\mathbb E_{h\sim\mathcal P}
\mathbb E_{S\sim D^m}[\Delta_S(h)^2].
\end{equation}
For any fixed $h$, since $\mathbb E_S[\Delta_S(h)]=0$, we have
\begin{equation}
\mathbb E_{S\sim D^m}[\Delta_S(h)^2]
=
\operatorname{Var}_{S\sim D^m}(R_S(h)).
\end{equation}
Moreover,
\begin{equation}
R_S(h)=\frac{1}{m}\sum_{i=1}^m \ell(h,z_i).
\end{equation}
Since $z_1,\ldots,z_m$ are independent and identically distributed,
\begin{equation}
\operatorname{Var}_{S\sim D^m}(R_S(h))
=
\frac{1}{m^2}
\sum_{i=1}^m
\operatorname{Var}_{z_i\sim D}(\ell(h,z_i)).
\end{equation}
By the bounded-variance assumption, there exists a constant $\sigma^2>0$ such that, for all $h\in\mathcal H$,
\begin{equation}
\operatorname{Var}_{z\sim D}(\ell(h,z))\le \sigma^2.
\end{equation}
Therefore,
\begin{equation}
\operatorname{Var}_{S\sim D^m}(R_S(h))
\le
\frac{\sigma^2}{m}.
\end{equation}
Consequently,
\begin{equation}
\mathbb E_{S\sim D^m}[Z_S]
\le
\frac{\sigma^2}{m}.
\end{equation}

Since $Z_S\ge 0$, Markov's inequality gives
\begin{equation}
\Pr_{S\sim D^m}
\left(
Z_S
\ge
\frac{\sigma^2}{m\delta}
\right)
\le
\frac{\mathbb E_S[Z_S]}{\sigma^2/(m\delta)}
\le
\delta.
\end{equation}
Thus, with probability at least $1-\delta$,
\begin{equation}
Z_S
\le
\frac{\sigma^2}{m\delta}.
\end{equation}
Substituting this bound into the previous Cauchy--Schwarz inequality yields that, with probability at least $1-\delta$, for all $\mathcal Q\ll\mathcal P$,
\begin{equation}
R_D(\mathcal Q)-R_S(\mathcal Q)
\le
\sqrt{\chi^2(\mathcal Q\|\mathcal P)+1}
\sqrt{\frac{\sigma^2}{m\delta}}.
\end{equation}
Equivalently,
\begin{equation}
R_D(\mathcal Q)
\le
R_S(\mathcal Q)
+
\sqrt{
\frac{\sigma^2}{m\delta}
\left(
\chi^2(\mathcal Q\|\mathcal P)+1
\right)
}.
\end{equation}
This completes the proof.
\end{proof}

\subsection{Proof of Corollary 2.2}\label{C2.2}

\begin{corollary}[Auditing Generalization Bound for Membership Inference]
Let $\ell: \mathcal{H} \times \mathcal{Z} \to [0,1]$ be the loss function of a membership inference attacker. Let $\mathcal{D}_f$ denote the true auditing distribution, and let $\mathcal{D}_t$ denote the training distribution used by the attacker. Assume $\mathcal{D}_t \ll \mathcal{D}_f$ and define the Rényi-$\infty$ divergence as
\begin{equation}
    D_\infty(\mathcal{D}_t \| \mathcal{D}_f) = \log \left( \operatorname*{ess\,sup}_{z \sim \mathcal{D}_f} \frac{d\mathcal{D}_t}{d\mathcal{D}_f}(z) \right).
\end{equation}
Let $S = (z_1, \dots, z_m) \sim \mathcal{D}_t^m$ be an i.i.d. training sample. Then, for any prior $\mathcal{P}$ over $\mathcal{H}$ and any posterior $\mathcal{Q} \ll \mathcal{P}$, with probability at least $1 - \delta$ over $S$, the true auditing risk satisfies
\begin{equation}
    R_{\mathcal{D}_f}(\mathcal{Q}) \leq R_S(\mathcal{Q}) + \sqrt{ \frac{2}{m\delta} \left( \chi^2(\mathcal{Q} \| \mathcal{P}) + 1 \right)  } + \sqrt{ \frac{1}{2} D_\infty(\mathcal{D}_t \| \mathcal{D}_f) }.
\end{equation}
\end{corollary}

\begin{proof}
Let $S=(z_1,\ldots,z_m)\sim \mathcal F_t^m$ be the training sample used to learn the MIA attacker. For any distribution $\mathcal F$, define
\begin{equation}
R_{\mathcal F}(\mathcal Q)
=
\mathbb E_{h\sim\mathcal Q}
\mathbb E_{z\sim\mathcal F}[\ell(h,z)].
\end{equation}
We decompose the auditing risk as
\begin{equation}
R_{\mathcal F_u}(\mathcal Q)-R_S(\mathcal Q)
=
\left(
R_{\mathcal F_t}(\mathcal Q)-R_S(\mathcal Q)
\right)
+
\left(
R_{\mathcal F_u}(\mathcal Q)-R_{\mathcal F_t}(\mathcal Q)
\right).
\end{equation}

The first term is the standard PAC-Bayesian statistical error under the attacker training distribution $\mathcal F_t$. By applying the PAC-Bayesian bound with $\chi^2$-divergence under a bounded-variance loss to $\mathcal F_t$, with probability at least $1-\delta$ over the draw of $S\sim\mathcal F_t^m$, for all posteriors $\mathcal Q\ll\mathcal P$,
\begin{equation}
R_{\mathcal F_t}(\mathcal Q)
\le
R_S(\mathcal Q)
+
\sqrt{
\frac{\sigma^2}{m\delta}
\left(
\chi^2(\mathcal Q\|\mathcal P)+1
\right)
}.
\end{equation}

It remains to control the second term, which measures the auditing error caused by the distribution shift from $\mathcal F_t$ to $\mathcal F_u$. Since the loss is bounded in $[0,1]$, for any fixed hypothesis $h\in\mathcal H$,
\begin{equation}
\left|
R_{\mathcal F_u}(h)-R_{\mathcal F_t}(h)
\right|
\le
\operatorname{TV}(\mathcal F_t,\mathcal F_u).
\end{equation}
Taking expectation over $h\sim\mathcal Q$ gives
\begin{equation}
R_{\mathcal F_u}(\mathcal Q)-R_{\mathcal F_t}(\mathcal Q)
\le
\operatorname{TV}(\mathcal F_t,\mathcal F_u).
\end{equation}

We now bound the total variation distance by the Rényi divergence of order infinity. Let
\begin{equation}
\rho(z)
=
\frac{d\mathcal F_t}{d\mathcal F_u}(z)
\end{equation}
be the density ratio between the attacker training distribution $\mathcal F_t$ and the audited distribution $\mathcal F_u$, and define
\begin{equation}
\varepsilon
=
D_\infty(\mathcal F_t\|\mathcal F_u)
=
\log\left(
\operatorname*{ess\,sup}_{z}\rho(z)
\right).
\end{equation}
If $\varepsilon=\infty$, then the auditing error term on the right-hand side diverges, and the desired inequality holds trivially. This case indicates the absence of finite density-ratio control between $\mathcal F_t$ and $\mathcal F_u$: the attacker training distribution is not sufficiently covered by the audited distribution, or the two distributions have insufficient overlap for this divergence-based auditing guarantee to be finite. Therefore, it suffices to consider the nontrivial case where $\varepsilon<\infty$.

When $\varepsilon<\infty$, we have
\begin{equation}
\rho(z)
\le
e^\varepsilon
\end{equation}
for $\mathcal F_u$-almost every $z$. Let
\begin{equation}
A
=
\left\{
z:
\rho(z)\ge 1
\right\}.
\end{equation}
By the variational characterization of total variation distance,
\begin{equation}
\operatorname{TV}(\mathcal F_t,\mathcal F_u)
=
\mathcal F_t(A)-\mathcal F_u(A).
\end{equation}
Moreover, since $\rho(z)\le e^\varepsilon$, we have
\begin{equation}
\mathcal F_u(A)
=
\int_A d\mathcal F_u
\ge
e^{-\varepsilon}
\int_A \rho(z)\,d\mathcal F_u(z)
=
e^{-\varepsilon}\mathcal F_t(A).
\end{equation}
Therefore,
\begin{equation}
\operatorname{TV}(\mathcal F_t,\mathcal F_u)
\le
\left(1-e^{-\varepsilon}\right)\mathcal F_t(A)
\le
1-e^{-\varepsilon}.
\end{equation}
Using the elementary inequality
\begin{equation}
1-e^{-\varepsilon}
\le
\sqrt{\frac{\varepsilon}{2}},
\qquad
\varepsilon\ge 0,
\end{equation}
we obtain
\begin{equation}
\operatorname{TV}(\mathcal F_t,\mathcal F_u)
\le
\sqrt{
\frac{1}{2}
D_\infty(\mathcal F_t\|\mathcal F_u)
}.
\end{equation}
Consequently,
\begin{equation}
R_{\mathcal F_u}(\mathcal Q)-R_{\mathcal F_t}(\mathcal Q)
\le
\sqrt{
\frac{1}{2}
D_\infty(\mathcal F_t\|\mathcal F_u)
}.
\end{equation}

Combining the PAC-Bayesian statistical error under $\mathcal F_t$ with the auditing error bound yields that, with probability at least $1-\delta$, for all $\mathcal Q\ll\mathcal P$,
\begin{equation}
R_{\mathcal F_u}(\mathcal Q)
\le
R_S(\mathcal Q)
+
\sqrt{
\frac{\sigma^2}{m\delta}
\left(
\chi^2(\mathcal Q\|\mathcal P)+1
\right)
}
+
\sqrt{
\frac{1}{2}
D_\infty(\mathcal F_t\|\mathcal F_u)
}.
\end{equation}
This completes the proof.
\end{proof}
\subsection{Proof of Proposition 3.1} \label{P3.1}
\begin{proposition}
If $\mathcal F_u$ is modeled as a mixture of $\mathcal F_n$ and $\mathcal F_m$ with proportions $\rho$ and $1-\rho$, respectively, i.e.,
\begin{equation}
    \mathcal F_u=\rho \mathcal F_n +(1-\rho)\mathcal F_m,
\end{equation}
then its mean $\mu_u$ and covariance $\Sigma_u$ satisfy
\begin{equation}
        \mu_u=\rho \mu_n+(1-\rho)\mu_m,\quad
        \Sigma_u=\rho \Sigma_n+(1-\rho)\Sigma_m
        +(\rho-\rho^2)(\mu_n-\mu_m)(\mu_n-\mu_m)^\top.
\end{equation}
\end{proposition}
\begin{proof}[Proof of Proposition~\ref{P1}]
Let $x\sim\mathcal F_u$ be a random variable drawn from the audited mixture distribution. Since
\begin{equation}
\mathcal F_u=\rho\mathcal F_n+(1-\rho)\mathcal F_m,
\end{equation}
the mean of $\mathcal F_u$ is
\begin{align}
\mu_u
&=
\mathbb E_{x\sim\mathcal F_u}[x] \nonumber\\
&=
\rho \mathbb E_{x\sim\mathcal F_n}[x]
+
(1-\rho)\mathbb E_{x\sim\mathcal F_m}[x] \nonumber\\
&=
\rho\mu_n+(1-\rho)\mu_m .
\end{align}

Next, we derive the covariance of $\mathcal F_u$. By definition,
\begin{align}
\Sigma_u
&=
\mathbb E_{x\sim\mathcal F_u}
\left[
(x-\mu_u)(x-\mu_u)^\top
\right] \nonumber\\
&=
\rho\,
\mathbb E_{x\sim\mathcal F_n}
\left[
(x-\mu_u)(x-\mu_u)^\top
\right]
+
(1-\rho)\,
\mathbb E_{x\sim\mathcal F_m}
\left[
(x-\mu_u)(x-\mu_u)^\top
\right] \nonumber\\
&=
\rho\,
\mathbb E_{x\sim\mathcal F_n}
\left[
(x-\mu_n+\mu_n-\mu_u)
(x-\mu_n+\mu_n-\mu_u)^\top
\right] \nonumber\\
&\quad
+
(1-\rho)\,
\mathbb E_{x\sim\mathcal F_m}
\left[
(x-\mu_m+\mu_m-\mu_u)
(x-\mu_m+\mu_m-\mu_u)^\top
\right] \nonumber\\
&=
\rho\Sigma_n
+
(1-\rho)\Sigma_m
+
\rho(\mu_n-\mu_u)(\mu_n-\mu_u)^\top
+
(1-\rho)(\mu_m-\mu_u)(\mu_m-\mu_u)^\top .
\end{align}
Using
\begin{equation}
\mu_u=\rho\mu_n+(1-\rho)\mu_m,
\end{equation}
we have
\begin{equation}
\mu_n-\mu_u
=
(1-\rho)(\mu_n-\mu_m),
\qquad
\mu_m-\mu_u
=
-\rho(\mu_n-\mu_m).
\end{equation}
Substituting these identities into the covariance expression gives
\begin{align}
\Sigma_u
&=
\rho\Sigma_n
+
(1-\rho)\Sigma_m
+
\rho(1-\rho)^2
(\mu_n-\mu_m)(\mu_n-\mu_m)^\top \nonumber\\
&\quad
+
(1-\rho)\rho^2
(\mu_n-\mu_m)(\mu_n-\mu_m)^\top \nonumber\\
&=
\rho\Sigma_n
+
(1-\rho)\Sigma_m
+
\left[
\rho(1-\rho)^2+(1-\rho)\rho^2
\right]
(\mu_n-\mu_m)(\mu_n-\mu_m)^\top \nonumber\\
&=
\rho\Sigma_n
+
(1-\rho)\Sigma_m
+
\rho(1-\rho)
(\mu_n-\mu_m)(\mu_n-\mu_m)^\top \nonumber\\
&=
\rho\Sigma_n
+
(1-\rho)\Sigma_m
+
(\rho-\rho^2)
(\mu_n-\mu_m)(\mu_n-\mu_m)^\top .
\end{align}
This completes the proof.
\end{proof}
\subsection{Proof of Lemma 3.2} \label{L3.2}
\begin{lemma}[Solution of SMI]
Let
$d=\mu_n-\mu_m,B=dd^\top,
A=\Sigma_u-\Sigma_m,
G=\Sigma_n-\Sigma_m+B .$
SMI estimates the mixture ratio by solving:
\begin{equation}
\min_{0\le\rho\le 1}\left\|
\Sigma_u-
\left[
\rho \Sigma_n+(1-\rho)\Sigma_m
+(\rho-\rho^2)(\mu_n-\mu_m)(\mu_n-\mu_m)^\top
\right]
\right\|^2_F
\end{equation}
The solution is selected from the search space $\mathcal C$:
\begin{equation}
\rho^\star=
\arg\min_{\rho\in\mathcal C}
\left\|A-\rho G+\rho^2 B\right\|_F^2, \label{SMI-J}
\end{equation}
The search space $\mathcal C$ is:
\begin{equation}
\mathcal C
=
\{0,1\}
\cup
\Big\{
\rho\in[0,1]:
2\|B\|_F^2\rho^3
-3\langle G,B\rangle_F\rho^2 +
\left(\|G\|_F^2+2\langle A,B\rangle_F\right)\rho
-\langle A,G\rangle_F
=0
\Big\},
\end{equation}
where, \(\langle A,B\rangle_F=\mathrm{tr}(A^\top B)\) denotes the Frobenius inner product. Specifically, the solution process can be described as first solving a univariate cubic equation in $\rho$, and then substituting the obtained roots into the search space to select the minimizer.
\end{lemma}
\begin{proof}[Proof of Lemma~\ref{lemmaSMI}]
By Proposition~\ref{P1}, the covariance predicted by the mixture model with mixture ratio $\rho$ is
\begin{equation}
\widehat{\Sigma}_u(\rho)
=
\rho \Sigma_n+(1-\rho)\Sigma_m
+
(\rho-\rho^2)(\mu_n-\mu_m)(\mu_n-\mu_m)^\top.
\end{equation}
Let
\begin{equation}
d=\mu_n-\mu_m,
\qquad
B=dd^\top,
\qquad
A=\Sigma_u-\Sigma_m,
\qquad
G=\Sigma_n-\Sigma_m+B.
\end{equation}
Then
\begin{equation}
\widehat{\Sigma}_u(\rho)
=
\Sigma_m+\rho(\Sigma_n-\Sigma_m+B)-\rho^2B
=
\Sigma_m+\rho G-\rho^2B.
\end{equation}
Therefore, the covariance matching objective can be rewritten as
\begin{equation}
\left\|
\Sigma_u-\widehat{\Sigma}_u(\rho)
\right\|_F^2
=
\left\|
A-\rho G+\rho^2B
\right\|_F^2.
\end{equation}
Thus, SMI estimates $\rho$ by solving the one-dimensional constrained optimization problem
\begin{equation}
\rho^\star
=
\arg\min_{0\le \rho\le 1}
J(\rho),
\qquad
J(\rho)
=
\left\|
A-\rho G+\rho^2B
\right\|_F^2.
\end{equation}

We now derive the candidate set for the global minimizer. Using the Frobenius inner product
\begin{equation}
\langle X,Y\rangle_F=\mathrm{tr}(X^\top Y),
\end{equation}
the objective can be expanded as
\begin{equation}
\begin{aligned}
J(\rho)
&=
\left\langle A-\rho G+\rho^2B,
A-\rho G+\rho^2B
\right\rangle_F \\
&=
\|A\|_F^2
-2\rho\langle A,G\rangle_F
+\rho^2\|G\|_F^2
+2\rho^2\langle A,B\rangle_F \\
&\quad
-2\rho^3\langle G,B\rangle_F
+\rho^4\|B\|_F^2 .
\end{aligned}
\end{equation}
Since $J(\rho)$ is a continuous polynomial on the compact interval $[0,1]$, it attains a global minimizer. Any global minimizer must either lie at the boundary points $\rho=0$ or $\rho=1$, or be an interior stationary point satisfying $J'(\rho)=0$.

Differentiating $J(\rho)$ gives
\begin{equation}
\begin{aligned}
J'(\rho)
&=
-2\langle A,G\rangle_F
+2\rho\|G\|_F^2
+4\rho\langle A,B\rangle_F \\
&\quad
-6\rho^2\langle G,B\rangle_F
+4\rho^3\|B\|_F^2 .
\end{aligned}
\end{equation}
Dividing by $2$, the stationary condition $J'(\rho)=0$ is equivalent to
\begin{equation}
2\|B\|_F^2\rho^3
-3\langle G,B\rangle_F\rho^2
+
\left(
\|G\|_F^2+2\langle A,B\rangle_F
\right)\rho
-
\langle A,G\rangle_F
=
0.
\end{equation}
Therefore, all possible global minimizers are contained in the finite candidate set
\begin{equation}
\mathcal C
=
\{0,1\}
\cup
\Big\{
\rho\in[0,1]:
2\|B\|_F^2\rho^3
-3\langle G,B\rangle_F\rho^2
+
\left(\|G\|_F^2+2\langle A,B\rangle_F\right)\rho
-\langle A,G\rangle_F
=0
\Big\}.
\end{equation}
Consequently, the SMI estimate is obtained by evaluating the objective on this candidate set and selecting the minimizer:
\begin{equation}
\rho^\star
=
\arg\min_{\rho\in\mathcal C}
\left\|
A-\rho G+\rho^2B
\right\|_F^2.
\end{equation}
This completes the proof.
\end{proof}

\subsection{Proof of Lemma 3.3}\label{A:3.3}
\begin{lemma}
    For the embedded $\mu^{(k)}_f, \mu^{(k)}_v, \mu^{(k)}_t$, the optimization problem~\eqref{youhauwentismi} is equivalent to the following convex quadratic programming problem:
    \begin{equation}
        \min_{0\le\alpha\le 1}\,\, \alpha^2\|\mu^{(k)}_v-\mu^{(k)}_t \|^2_\mathcal H-2\alpha\langle\mu_f^{(k)}-\mu_t^{(k)}, \mu_v^{(k)}-\mu_t^{(k)}\rangle_\mathcal H
    \end{equation}
    This provides an efficient solution path for the optimization problem~\eqref{heyouhua}.
\end{lemma}

\begin{proof}
We begin by expanding the squared Hilbert space norm in the objective of problem~\eqref{youhauwentismi}. Recall that for any $a, b \in \mathcal{H}$, we have $\|a - b\|_{\mathcal{H}}^2 = \langle a - b, a - b \rangle_{\mathcal{H}}$. Applying this to the objective function:

\begin{equation}
 \begin{aligned}
    &\|\mu^{(k)}_f - \alpha\mu^{(k)}_v - (1-\alpha)\mu^{(k)}_t \|_{\mathcal{H}}^2 \nonumber \\
    &= \big\langle \mu^{(k)}_f - \alpha\mu^{(k)}_v - (1-\alpha)\mu^{(k)}_t,\,
                    \mu^{(k)}_f - \alpha\mu^{(k)}_v - (1-\alpha)\mu^{(k)}_t \big\rangle_{\mathcal{H}}. \end{aligned}
\end{equation}

We simplify the expression inside the inner product. Observe that:
\begin{equation}
    \alpha\mu^{(k)}_v + (1-\alpha)\mu^{(k)}_t = \mu^{(k)}_t + \alpha(\mu^{(k)}_v - \mu^{(k)}_t).
\end{equation}
Therefore,
\begin{equation}
 \mu^{(k)}_f - \alpha\mu^{(k)}_v - (1-\alpha)\mu^{(k)}_t = (\mu^{(k)}_f - \mu^{(k)}_t) - \alpha(\mu^{(k)}_v - \mu^{(k)}_t).
\end{equation}

Let us denote:
\begin{equation}
A := \mu^{(k)}_f - \mu^{(k)}_t, \quad B := \mu^{(k)}_v - \mu^{(k)}_t.
\end{equation}

Then the objective becomes:
\begin{equation}
\|A - \alpha B\|_{\mathcal{H}}^2 = \langle A - \alpha B, A - \alpha B \rangle_{\mathcal{H}}.
\end{equation}

Expanding the inner product using bilinearity and symmetry:
\begin{equation}
\begin{aligned}
    \|A - \alpha B\|_{\mathcal{H}}^2
    &= \langle A, A \rangle_{\mathcal{H}} - 2\alpha \langle A, B \rangle_{\mathcal{H}} + \alpha^2 \langle B, B \rangle_{\mathcal{H}} \nonumber \\
    &= \|A\|_{\mathcal{H}}^2 - 2\alpha \langle A, B \rangle_{\mathcal{H}} + \alpha^2 \|B\|_{\mathcal{H}}^2.
\end{aligned}
\end{equation}

Since $\|A\|_{\mathcal{H}}^2$ is independent of $\alpha$, it does not affect the minimizer of the optimization problem over $\alpha \in [0,1]$. Therefore, minimizing the original objective is equivalent to minimizing the function:
\begin{equation}
\alpha^2 \|B\|_{\mathcal{H}}^2 - 2\alpha \langle A, B \rangle_{\mathcal{H}}.
\end{equation}

Substituting back the definitions of $A$ and $B$, we obtain:
\begin{equation}
\alpha^2 \|\mu^{(k)}_v - \mu^{(k)}_t\|_{\mathcal{H}}^2 - 2\alpha \langle \mu^{(k)}_f - \mu^{(k)}_t,\, \mu^{(k)}_v - \mu^{(k)}_t \rangle_{\mathcal{H}},
\end{equation}
which is precisely the objective in equation~\eqref{youhuamd}.

Finally, note that the objective in~\eqref{youhuamd} is a quadratic function in $\alpha$ with a non-negative leading coefficient $\|\mu^{(k)}_v - \mu^{(k)}_t\|_{\mathcal{H}}^2 \geq 0$. Hence, it is convex in $\alpha$, and the constraint $0 \leq \alpha \leq 1$ defines a compact convex set. Therefore, problem~\eqref{youhuamd} is a convex quadratic programming problem, which admits a unique global minimizer that can be efficiently computed.

This completes the proof.
\end{proof}

\section{Theoretical Elaboration on SMI-M and MMD}\label{mmd}
In this section, we elucidate the motivation behind SMI-M. In traditional statistical inference and optimization problems, the similarity between two random variables or feature distributions is often measured by matching the first moment (mean) or the second moment (covariance matrix) of the samples. For instance, in the SMI-$\Sigma$ method, the objective is to align the empirical covariance $\widehat{\Sigma}_f$ of the auditing features $f(\cdot; w)$ with the target covariance $\Sigma^\ast$. However, relying solely on low-order moments presents significant limitations:
\begin{itemize}
    \item Low-order moments fail to characterize high-order structures of the distribution (such as skewness, kurtosis, and multimodality).
    \item In high-dimensional spaces, the estimation of covariance matrices is susceptible to noise and small sample sizes, leading to the "curse of dimensionality".
    \item If the true distribution is non-Gaussian, matching only the first two moments may completely overlook critical differences.
\end{itemize}
To overcome these issues, modern approaches tend to directly measure the \textbf{distance between entire probability distributions} rather than just their moments. Among these, the \textbf{Maximum Mean Discrepancy} (MMD) serves as a non-parametric distribution metric with a solid theoretical foundation and computational feasibility.

Let $\mathcal{X} \subseteq \mathbb{R}^d$ be a non-empty set (typically a subset of Euclidean space), and let $\mathbb{P}$ and $\mathbb{Q}$ be two Borel probability measures defined on $\mathcal{X}$. Let $k: \mathcal{X} \times \mathcal{X} \to \mathbb{R}$ be a \textbf{Mercer kernel function}, which satisfies the following conditions:
\begin{enumerate}
    \item Symmetry: $k(x, x') = k(x', x)$ for all $x, x' \in \mathcal{X}$;
    \item Positive Definiteness: For any finite set of points $\{x_1, \dots, x_n\} \subset \mathcal{X}$ and any real coefficients $\{a_i\}_{i=1}^n \subset \mathbb{R}$, the following holds:
    \[
        \sum_{i=1}^n \sum_{j=1}^n a_i a_j k(x_i, x_j) \geq 0.
    \]
\end{enumerate}
According to Mercer's theorem, there exists a unique \textbf{Reproducing Kernel Hilbert Space} (RKHS) $\mathcal{H}_k$, where the inner product $\langle \cdot, \cdot \rangle_{\mathcal{H}_k}$ satisfies the \textbf{reproducing property}:
\[
    \langle f, k(x, \cdot) \rangle_{\mathcal{H}_k} = f(x), \quad \forall f \in \mathcal{H}_k,\, x \in \mathcal{X}.
\]
Furthermore, the RKHS $\mathcal{H}_k$ can be constructed via a feature map $\phi: \mathcal{X} \to \mathcal{H}_k$, such that $k(x, x') = \langle \phi(x), \phi(x') \rangle_{\mathcal{H}_k}$. This mapping embeds the input space into a (potentially infinite-dimensional) Hilbert space.

Given a probability measure $\mathbb{P}$, its \textbf{kernel mean embedding} in the RKHS $\mathcal{H}_k$ is defined as the Bochner integral:
\[
    \mu_{\mathbb{P}}^{(k)} := \mathbb{E}_{x \sim \mathbb{P}}[\phi(x)] = \int_{\mathcal{X}} \phi(x) \, d\mathbb{P}(x) \in \mathcal{H}_k,
\]
provided that this expectation exists in $\mathcal{H}_k$ (e.g., this holds when $\mathbb{E}_{x \sim \mathbb{P}}[\sqrt{k(x,x)}] < \infty$). This embedding compresses the entire distribution $\mathbb{P}$ into a single vector within the RKHS, thereby transforming the problem of distribution comparison into a geometric problem within a Hilbert space.

The \textbf{Maximum Mean Discrepancy} (MMD) is defined as the Hilbert norm distance between the embeddings of the two distributions:
\begin{equation}
    \mathrm{MMD}_k(\mathcal{P}, \mathcal{Q})
    := \left\| \mu_{\mathcal{P}}^{(k)} - \mu_{\mathcal{Q}}^{(k)} \right\|_{\mathcal{H}_k}.
\end{equation}
Utilizing the reproducing property of the RKHS and the linearity of the inner product, the square of MMD can be expanded into a form dependent solely on the expectation of the kernel function:
\begin{align}
    \mathrm{MMD}_k^2(\mathcal{P}, \mathcal{Q})
    &= \left\langle \mu_{\mathcal{P}}^{(k)} - \mu_{\mathcal{Q}}^{(k)}, \mu_{\mathcal{P}}^{(k)} - \mu_{\mathcal{Q}}^{(k)} \right\rangle_{\mathcal{H}_k} \nonumber \\
    &= \mathcal{E}_{x, x' \sim \mathcal{P}}[k(x, x')]
      + \mathcal{E}_{y, y' \sim \mathcal{Q}}[k(y, y')]
      - 2\mathcal{E}_{x \sim \mathcal{P},\, y \sim \mathcal{Q}}[k(x, y)].
\end{align}
This expression demonstrates that MMD relies only on pairwise kernel values between samples, without the need to explicitly compute high-dimensional feature maps.

It is worth noting that whether MMD constitutes a valid \textbf{metric} depends on the properties of the kernel function $k$. If $k$ is a \textbf{characteristic kernel}, satisfying
\[
    \mathrm{MMD}_k(\mathcal{P}, \mathcal{Q}) = 0 \iff \mathcal{P} = \mathcal{Q},
\]
then MMD induces a metric on the space of probability measures. Sufficient conditions include: $k$ is continuous, bounded, and the corresponding RKHS $\mathcal{H}_k$ is dense in $\mathcal{X}$ (e.g., the Gaussian RBF kernel satisfies this property on compact sets).

\subsection{Empirical Estimation}
In practical applications, assuming we only have independent and identically distributed (i.i.d.) samples $\{x_i\}_{i=1}^n$ from $\mathcal{P}$ and samples $\{y_j\}_{j=1}^m$ from $\mathcal{Q}$, MMD can be estimated via empirical mean embeddings. Define the empirical embeddings:
\[
    \hat{\mu}_{\mathcal{P}}^{(k)} = \frac{1}{n} \sum_{i=1}^n \phi(x_i), \quad
    \hat{\mu}_{\mathcal{Q}}^{(k)} = \frac{1}{m} \sum_{j=1}^m \phi(y_j).
\]
The \textbf{biased estimator} is given by:
\begin{equation}
    \widehat{\mathrm{MMD}}_{k,\text{biased}}^2 = \left\| \hat{\mu}_{\mathcal{P}}^{(k)} - \hat{\mu}_{\mathcal{Q}}^{(k)} \right\|_{\mathcal{H}_k}^2
    = \frac{1}{n^2}\sum_{i=1}^n\sum_{j=1}^n k(x_i, x_j)
    + \frac{1}{m^2}\sum_{i=1}^m\sum_{j=1}^m k(y_i, y_j)
    - \frac{2}{nm}\sum_{i=1}^n\sum_{j=1}^m k(x_i, y_j).
\end{equation}
This estimator possesses a systematic bias under finite samples but is consistent (i.e., it converges in probability to the true MMD as $n,m \to \infty$).

To eliminate the bias from auto-correlation terms, an \textbf{unbiased estimator} can be adopted:
\begin{align}
    \widehat{\mathrm{MMD}}_{k,\text{unbiased}}^2
    &= \frac{1}{n(n-1)} \sum_{i \neq j} k(x_i, x_j)
    + \frac{1}{m(m-1)} \sum_{i \neq j} k(y_i, y_j)
    - \frac{2}{nm} \sum_{i=1}^n \sum_{j=1}^m k(x_i, y_j),
\end{align}
where the first and second terms exclude the diagonal entries where $i=j$. This estimator is unbiased when $n, m \geq 2$, although it may exhibit larger variance with small sample sizes.
\subsection{Common Kernel Functions and Calculation Methods}

The empirical performance of MMD depends on the choice of kernel function $k(x,x')$. For distance-based kernels, the \emph{bandwidth} or scale parameter, usually denoted by $\sigma>0$, controls how rapidly the kernel value decreases as the distance between two samples increases. A smaller bandwidth makes the kernel more sensitive to local differences, whereas a larger bandwidth produces smoother similarity scores. In practice, $\sigma$ is commonly selected by the median heuristic or cross-validation.

Commonly used kernel functions can be computed as follows:

\begin{enumerate}
    \item \textbf{Gaussian RBF Kernel}
    \[
        k_{\mathrm{RBF}}(x,x') =
        \exp\left(-\frac{\|x-x'\|_2^2}{2\sigma^2}\right).
    \]
    To compute this kernel, first calculate the squared Euclidean distance
    \[
        \|x-x'\|_2^2 = \sum_{i=1}^d (x_i-x_i')^2,
    \]
    then substitute it into the exponential function. The bandwidth $\sigma$ determines the scale at which two samples are regarded as similar.

    \item \textbf{Laplacian Kernel}
    \[
        k_{\mathrm{Lap}}(x,x') =
        \exp\left(-\frac{\|x-x'\|_1}{\sigma}\right),
    \]
    where
    \[
        \|x-x'\|_1 = \sum_{i=1}^d |x_i-x_i'|.
    \]
    The calculation first sums the absolute coordinate-wise differences, then divides the result by $\sigma$ and applies the negative exponential.

    \item \textbf{Linear Kernel}
    \[
        k_{\mathrm{lin}}(x,x') = x^\top x'.
    \]
    This kernel is computed directly by taking the inner product between two vectors:
    \[
        x^\top x' = \sum_{i=1}^d x_i x_i'.
    \]

    \item \textbf{Polynomial Kernel}
    \[
        k_{\mathrm{poly}}(x,x') = (x^\top x' + c)^p,
    \]
    where $c \geq 0$ is the bias term and $p \in \mathbb{N}$ is the polynomial degree. Its computation consists of first evaluating the inner product $x^\top x'$, adding the bias term $c$, and then raising the result to the power $p$.

    \item \textbf{Cosine Kernel}
    \[
        k_{\mathrm{cos}}(x,x') =
        \frac{x^\top x'}{\|x\|_2 \|x'\|_2}.
    \]
    This kernel measures the cosine similarity between two nonzero vectors. It is computed by taking the inner product of $x$ and $x'$, then normalizing it by the product of their Euclidean norms.
\end{enumerate}

After choosing a kernel, the empirical MMD estimator can be obtained by substituting sample pairs from the two distributions into the corresponding kernel function.
\newpage
\section{The Complete Experimental}\label{A:zhubiao}
The complete experimental results are shown in Table~\ref{tab:newreus}.
\begin{table}[h]
\centering
\caption{Performance auditing of SMI and MIA-based methods on the complete unlearning task. The Dataset column reports the model, task, unlearning percentage, and number of shadow models.Experimental results for full datasets, with the \textbf{best} and \underline{second best} results are highlighted. }
\label{tab:newreus}
\resizebox{\textwidth}{!}{%
\begin{tabular}{@{}llGGccccccc@{}}
\toprule
Dataset & Metric & SMI & SMI-M & RULI & IAM & RMIA & EMIA & LiRA & Unleak & Ramdom \\
\midrule
\multirow{3}{*}{Resnet18-Cifar10-5\%-64} & FNR & -- & -- & 13.28\% & 12.72\% & 49.68\% & 70.82\% & 13.80\% & 41.24\% & 49.52\% \\
 & FPR & -- & -- & 0.00\% & 13.08\% & 50.04\% & 70.72\% & 0.00\% & 43.48\% & 50.88\% \\
 & $\rho^*$ & \underline{88.89\%} & \textbf{92.44\%} & 87.94\% & 88.67\% & 66.68\% & 39.10\% & 87.84\% & 62.58\% & 50.06\% \\
\midrule
\multirow{3}{*}{Resnet18-Cifar10-10\%-64} & FNR & -- & -- & 19.38\% & 14.22\% & 56.98\% & 61.22\% & 18.20\% & 52.44\% & 49.88\% \\
 & FPR & -- & -- & 0.00\% & 18.28\% & 42.24\% & 67.83\% & 0.00\% & 31.88\% & 50.17\% \\
 & $\rho^*$ & 87.70\% & \textbf{91.80\%} & 83.24\% & 83.37\% & 61.68\% & 30.90\% & \underline{89.94\%} & 53.28\% & 50.02\% \\
\midrule
\multirow{3}{*}{Resnet18-Cifar100-5\%-16} & FNR & -- & -- & 15.60\% & 17.76\% & 87.80\% & 24.72\% & 10.47\% & 81.04\% & 48.72\% \\
 & FPR & -- & -- & 0.00\% & 17.96\% & 87.96\% & 22.24\% & 0.00\% & 82.88\% & 49.52\% \\
 & $\rho^*$ & 79.64\% & \textbf{89.84\%} & 83.00\% & \underline{89.36\%} & 67.19\% & 84.64\% & 86.64\% & 28.52\% & 50.88\% \\
\midrule
\multirow{3}{*}{Resnet18-Cifar100-10\%-16} & FNR & -- & -- & 31.00\% & 54.88\% & 58.96\% & 87.04\% & 32.24\% & 62.84\% & 49.71\% \\
 & FPR & -- & -- & 0.00\% & 52.12\% & 60.92\% & 87.00\% & 0.00\% & 64.36\% & 51.02\% \\
 & $\rho^*$ & 77.56\% & \textbf{91.84\%} & 87.48\% & 84.76\% & 81.74\% & 51.46\% & \underline{88.94\%} & 36.90\% & 50.58\% \\
\midrule
\multirow{3}{*}{Resnet50-Cifar100-5\%-32} & FNR & -- & -- & 36.76\% & 86.72\% & 91.92\% & 66.52\% & 37.41\% & 93.08\% & 47.68\% \\
 & FPR & -- & -- & 0.00\% & 63.83\% & 56.00\% & 57.56\% & 0.00\% & 47.69\% & 50.15\% \\
 & $\rho^*$ & 81.22\% & \textbf{92.64\%} & 73.08\% & 63.56\% & 58.24\% & 67.84\% & \underline{87.04\%} & 10.40\% & 54.68\% \\
\midrule
\multirow{3}{*}{Resnet50-Cifar100-10\%-32} & FNR & -- & -- & 26.90\% & 87.73\% & 91.50\% & 70.72\% & 38.29\% & 93.56\% & 49.71\% \\
 & FPR & -- & -- & 0.00\% & 61.69\% & 53.21\% & 57.81\% & 0.00\% & 49.77\% & 42.27\% \\
 & $\rho^*$ & 83.46\% & \textbf{93.95\%} & 74.44\% & 64.98\% & 60.74\% & 69.94\% & \underline{87.38\%} & 4.68\% & 54.14\% \\
\midrule
\multirow{3}{*}{Resnet18-Cinic10-5\%-16} & FNR & -- & -- & 28.44\% & 82.32\% & 90.64\% & 9.24\% & 5.36\% & 14.96\% & 49.92\% \\
 & FPR & -- & -- & 0.00\% & 81.68\% & 90.20\% & 8.36\% & 0.00\% & 14.52\% & 51.13\% \\
 & $\rho^*$ & 89.12\% & \textbf{92.10\%} & 88.47\% & 37.53\% & 29.44\% & 88.24\% & \underline{89.27\%} & 87.28\% & 50.05\% \\
\midrule
\multirow{3}{*}{Resnet18-Cinic10-10\%-16} & FNR & -- & -- & 24.17\% & 81.44\% & 94.08\% & 85.56\% & 17.64\% & 29.92\% & 49.52\% \\
 & FPR & -- & -- & 0.00\% & 76.84\% & 92.56\% & 78.84\% & 0\% & 22.04\% & 50.88\% \\
 & $\rho^*$ & 87.24\% & \textbf{92.41\%} & \underline{91.15\%} & 40.34\% & 24.98\% & 35.88\% & 89.62\% & 78.71\% & 49.83\% \\
\midrule
\multirow{3}{*}{Resnet50-Cinic10-5\%-16} & FNR & -- & -- & 11.68\% & 82.32\% & 90.64\% & 15.24\% & 11.36\% & 14.96\% & 49.24\% \\
 & FPR & -- & -- & 0.00\% & 81.68\% & 90.20\% & 15.36\% & 0.00\% & 14.52\% & 50.21\% \\
 & $\rho^*$ & 89.31\% & \textbf{90.72\%} & 89.74\% & 37.53\% & 29.44\% & 88.26\% & \underline{89.92\%} & 87.28\% & 49.36\% \\
\midrule
\multirow{3}{*}{Resnet50-Cinic10-10\%-16} & FNR & -- & -- & 11.36\% & 29.92\% & 11.88\% & 22.08\% & 11.18\% & 19.46\% & 49.82\% \\
 & FPR & -- & -- & 0.00\% & 30.96\% & 11.28\% & 21.68\% & 0.00\% & 14.04\% & 50.23\% \\
 & $\rho^*$ & \underline{90.89\%} & \textbf{92.27\%} & 89.88\% & 85.35\% & 87.64\% & 89.66\% & 90.47\% & 78.02\% & 49.76\% \\
\midrule
\multirow{3}{*}{VIT-Cinic10-5\%-16} & FNR & -- & -- & 15.68\% & 46.96\% & 15.56\% & 36.60\% & 19.27\% & 81.28\% & 49.42\% \\
 & FPR & -- & -- & 0.00\% & 48.16\% & 14.72\% & 35.36\% & 0.00\% & 81.96\% & 50.48\% \\
 & $\rho^*$ & \underline{89.72\%} & \textbf{92.13\%} & 85.05\% & 53.62\% & 85.30\% & 65.52\% & 78.62\% & 18.84\% & 50.52\% \\
\midrule
\multirow{3}{*}{VIT-Cinic10-10\%-16} & FNR & -- & -- & 15.92\% & 57.40\% & 85.04\% & 65.36\% & 18.76\% & 43.24\% & 49.97\% \\
 & FPR & -- & -- & 0.00\% & 53.37\% & 86.12\% & 65.24\% & 0.00\% & 44.76\% & 51.93\% \\
 & $\rho^*$ & \underline{91.44\%} & \textbf{93.62\%} & 89.84\% & 41.01\% & 15.33\% & 36.66\% & 83.51\% & 55.71\% & 51.04\% \\
\bottomrule
\end{tabular}
}
\end{table}
\subsection{Robustness Experiments}
We provide additional visualizations of the experiments in Figures~\ref{fig:a1} and~\ref{fig:A2}.
\begin{figure}[H]
\begin{subfigure}[t]{0.24\linewidth}
    \centering
    \includegraphics[width=\linewidth]{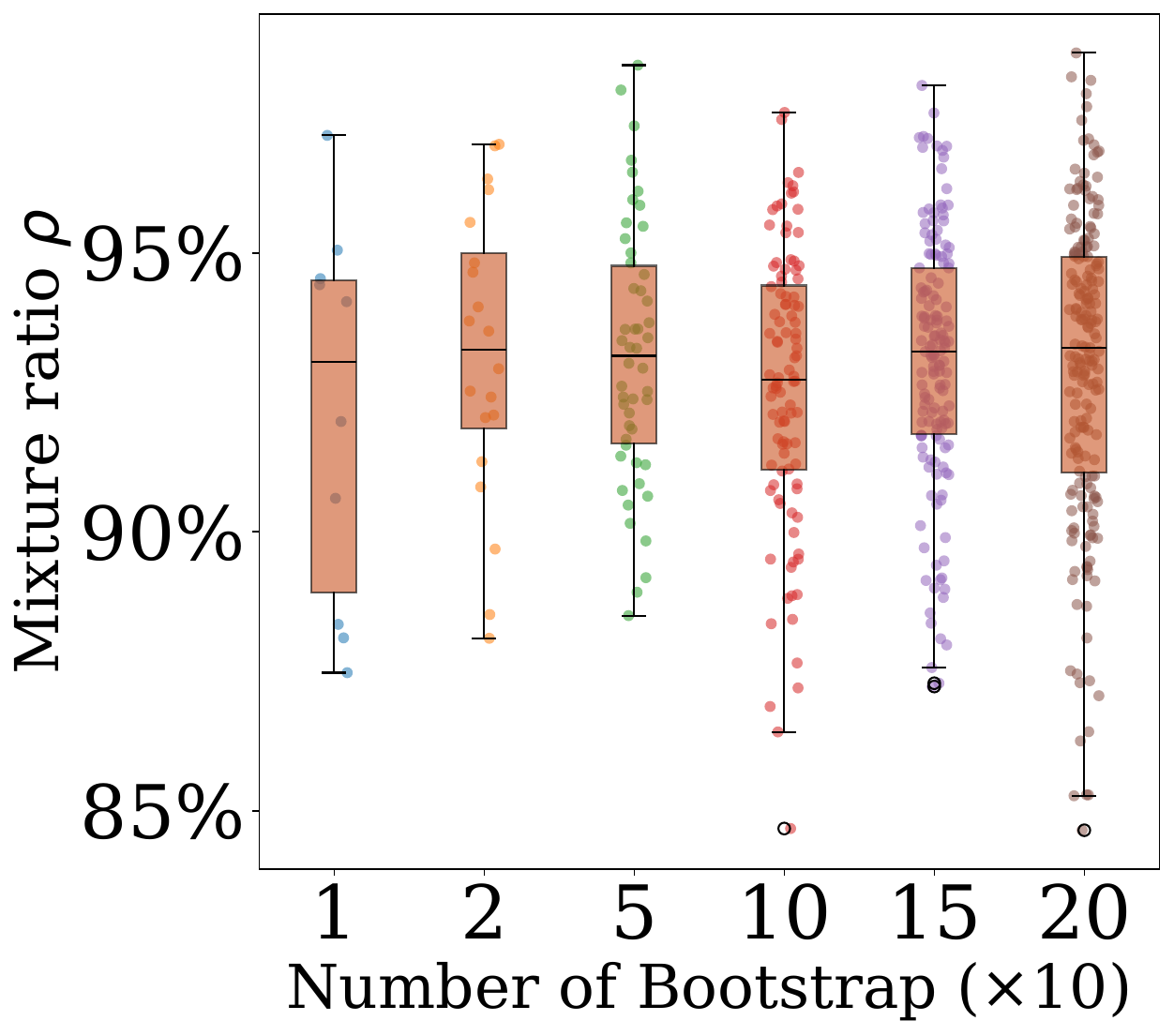}
    \caption{ResNet18-5\%-SMI}
    \label{fig:cinic-resnet50-a}
\end{subfigure}
\hfill
\begin{subfigure}[t]{0.24\linewidth}
    \centering
    \includegraphics[width=\linewidth]{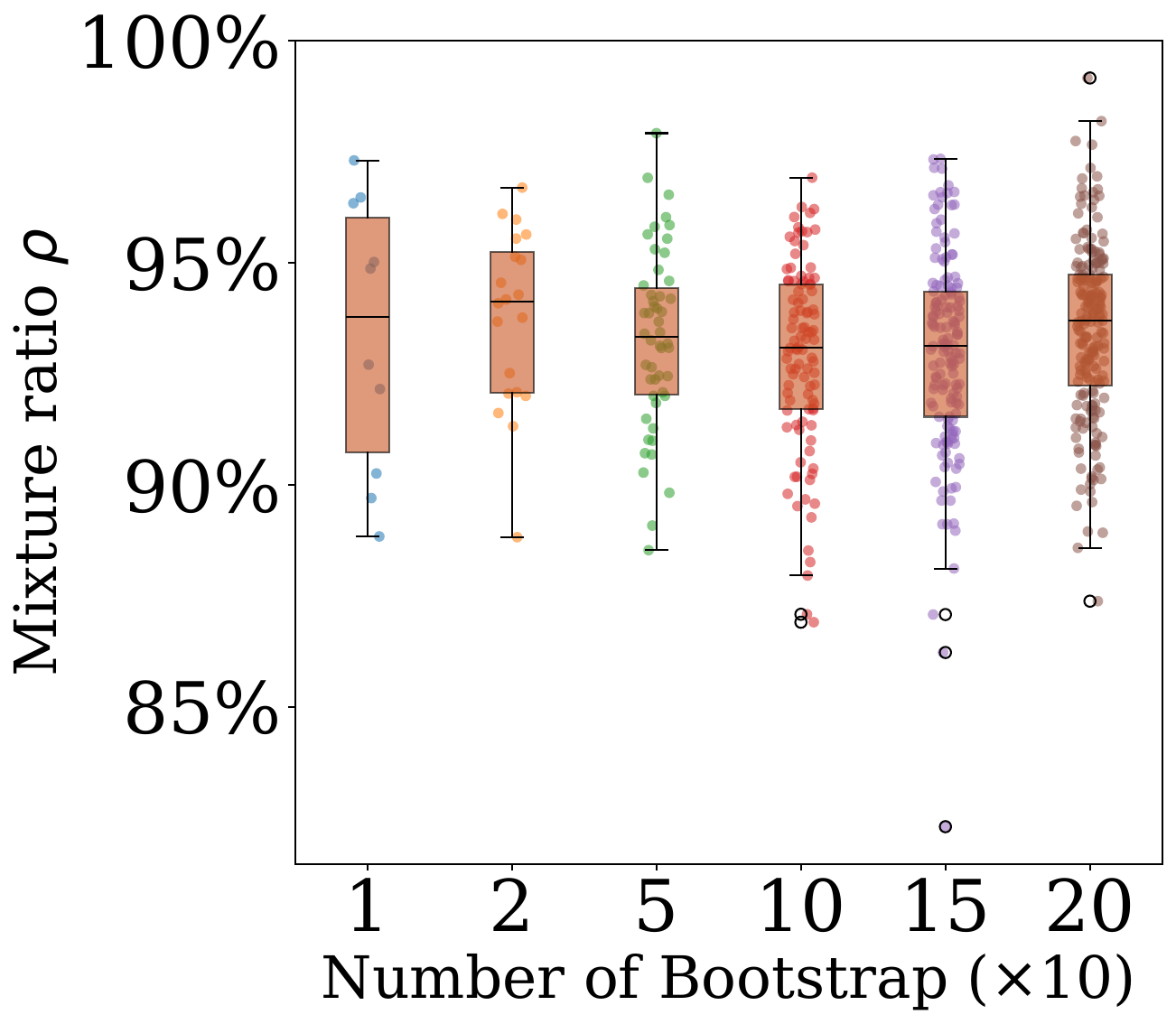}
    \caption{ResNet18-5\%-SMI-M}
    \label{fig:cinic-resnet50-b}
\end{subfigure}
\hfill
\begin{subfigure}[t]{0.24\linewidth}
    \centering
    \includegraphics[width=\linewidth]{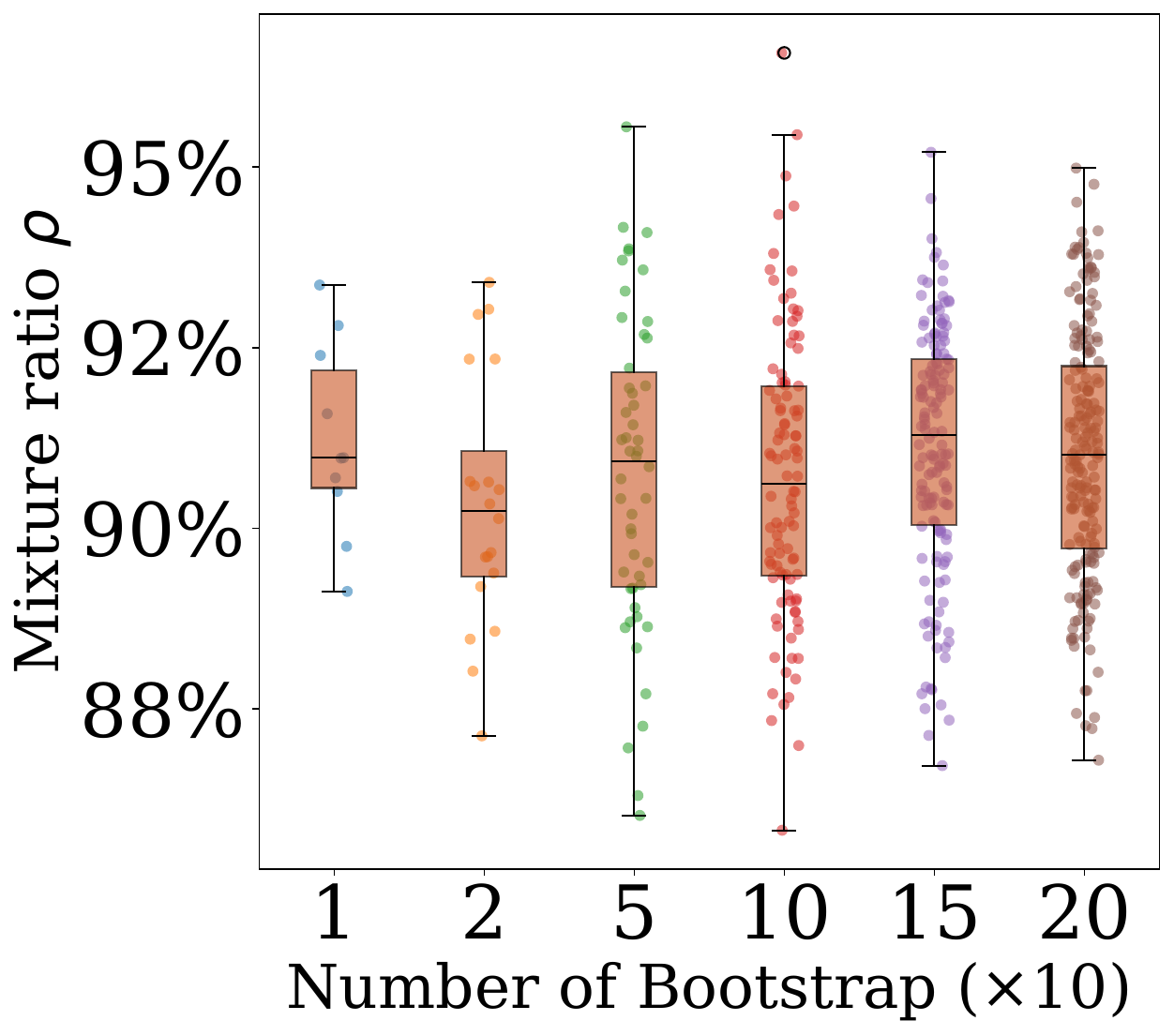}
    \caption{ResNet18-10\%-SMI}
    \label{fig:cinic-resnet50-b}
\end{subfigure}
\hfill
\begin{subfigure}[t]{0.24\linewidth}
    \centering
    \includegraphics[width=\linewidth]{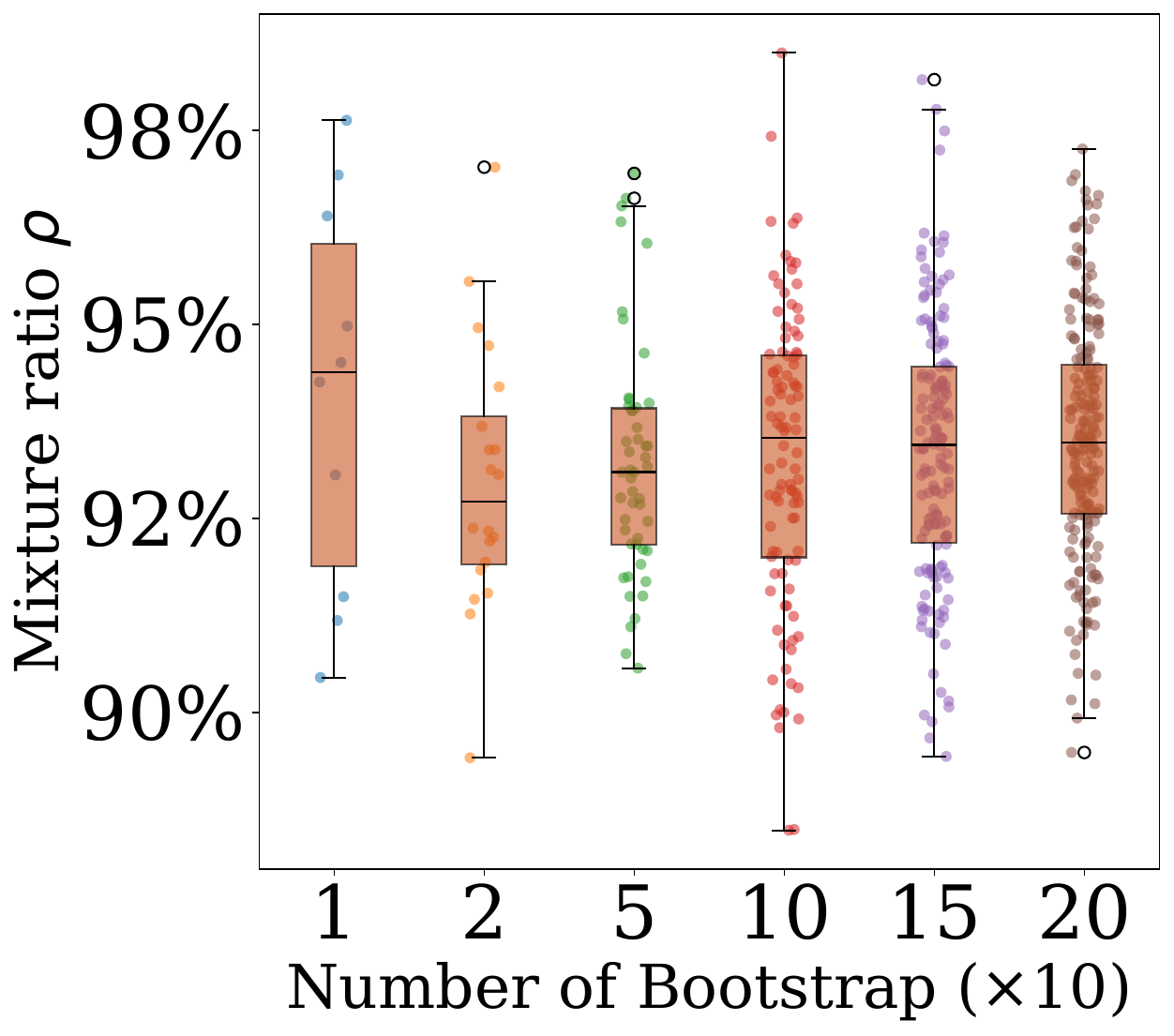}
    \caption{ResNet18-10\%-SMI-M}
    \label{fig:cinic-resnet50-b}
\end{subfigure}
\caption{Performance fluctuation of SMI and SMI-M with different numbers of resampling iterations on 5\% and 10\% randomly unlearned data.}
\label{fig:a1}
\end{figure}

\begin{figure}

\vspace{-1.0em}
\centering
\begin{subfigure}[t]{0.48\linewidth}
    \centering
    \includegraphics[width=\linewidth]{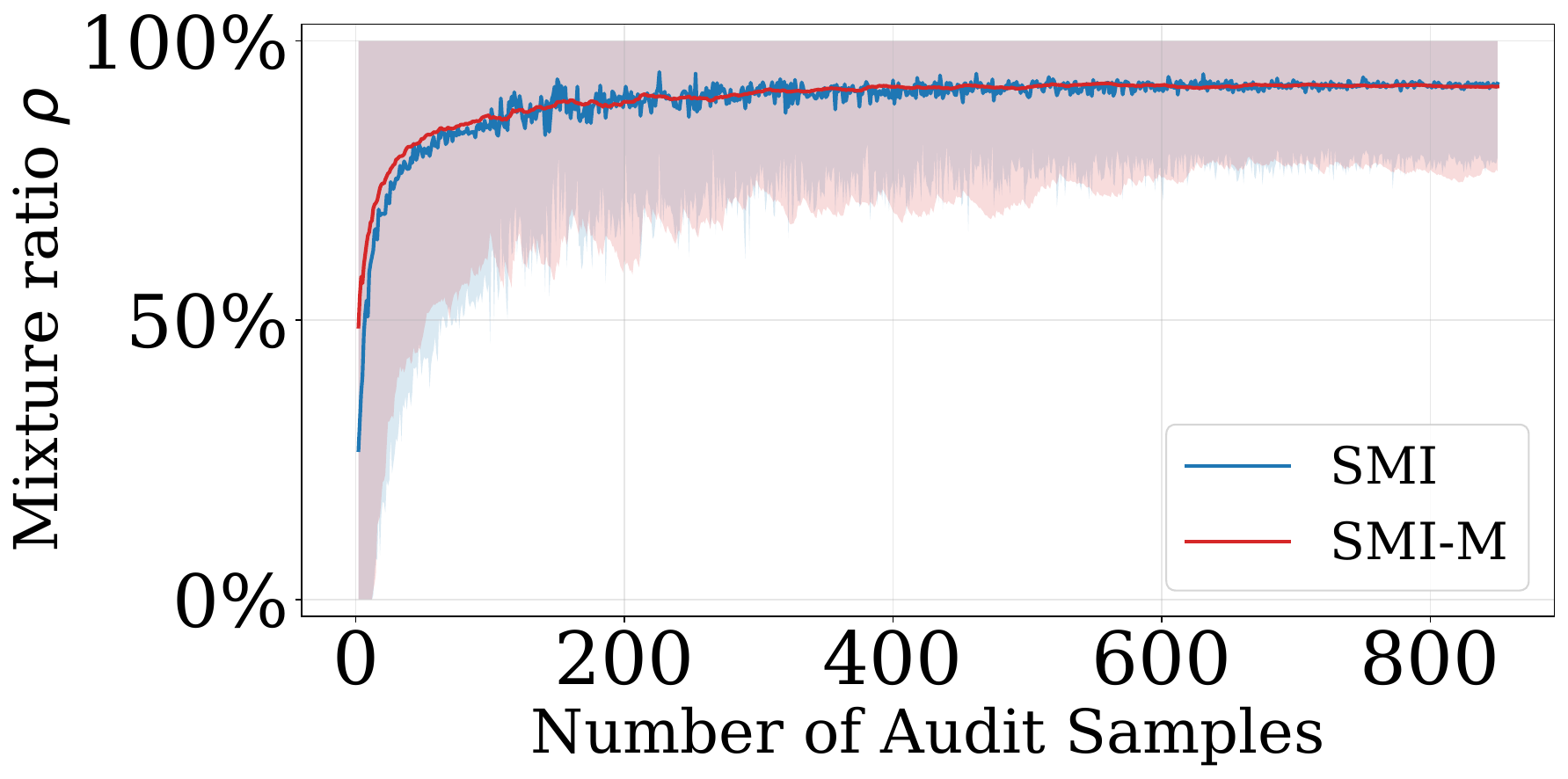}
    \caption{ResNet18-5\%}
    \label{fig:cinic-resnet50-a}
\end{subfigure}
\hfill
\begin{subfigure}[t]{0.48\linewidth}
    \centering
    \includegraphics[width=\linewidth]{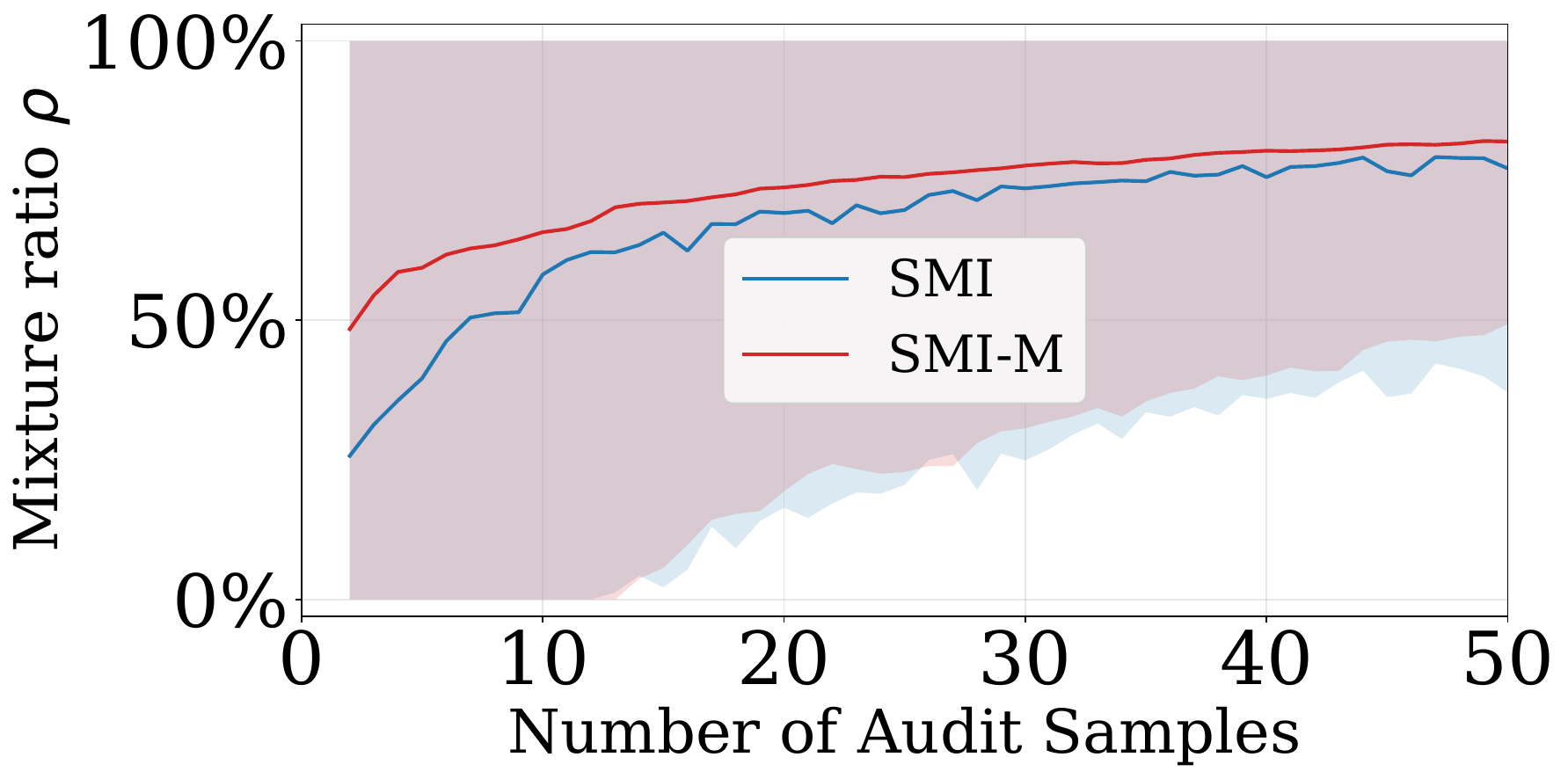}
    \caption{ResNet18-5\%}
    \label{fig:cinic-resnet50-b}
\end{subfigure}
\begin{subfigure}[t]{0.48\linewidth}
    \centering
    \includegraphics[width=\linewidth]{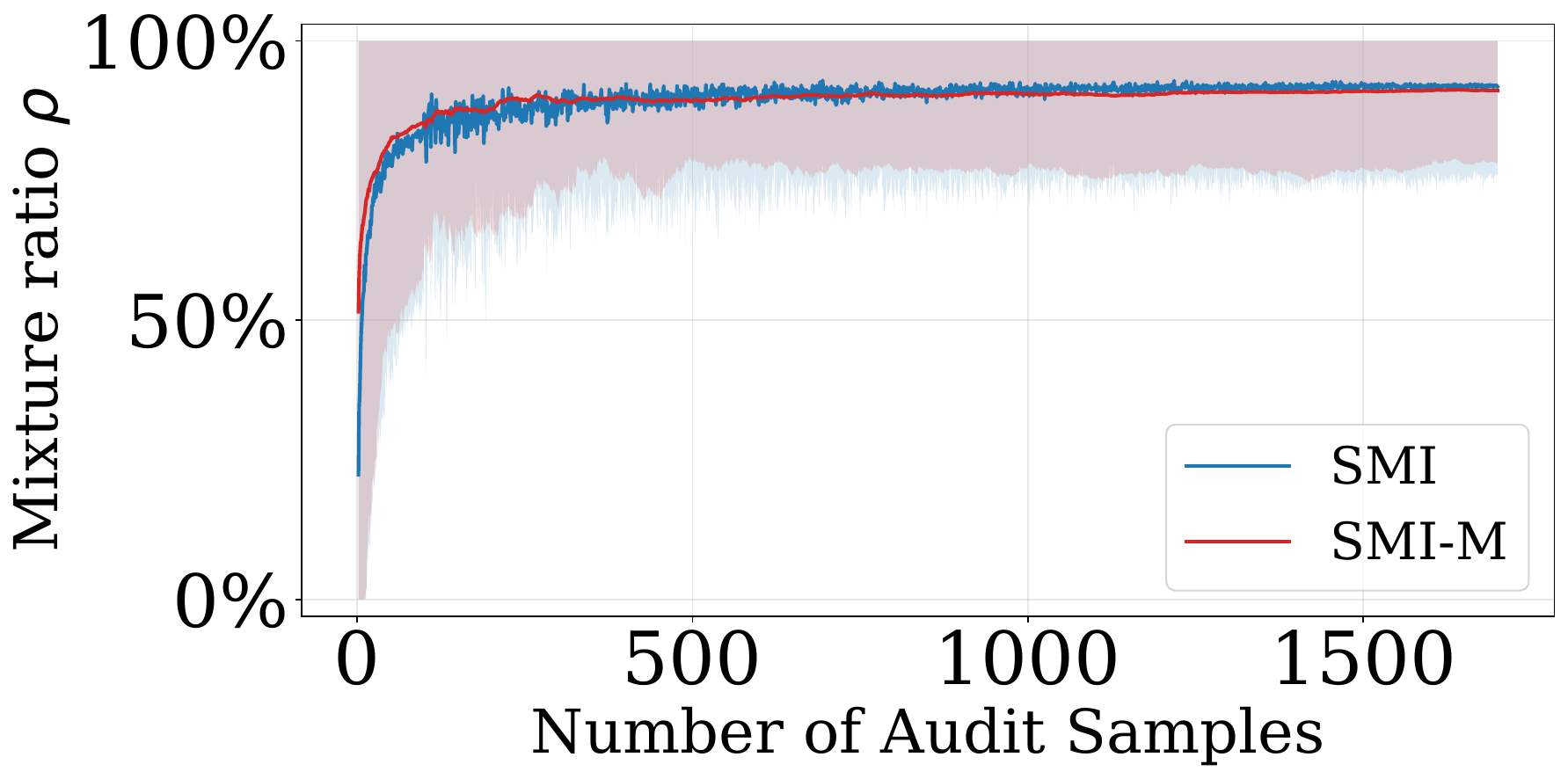}
    \caption{ResNet18-10\%}
    \label{fig:cinic-resnet50-a}
\end{subfigure}
\hfill
\begin{subfigure}[t]{0.48\linewidth}
    \centering
    \includegraphics[width=\linewidth]{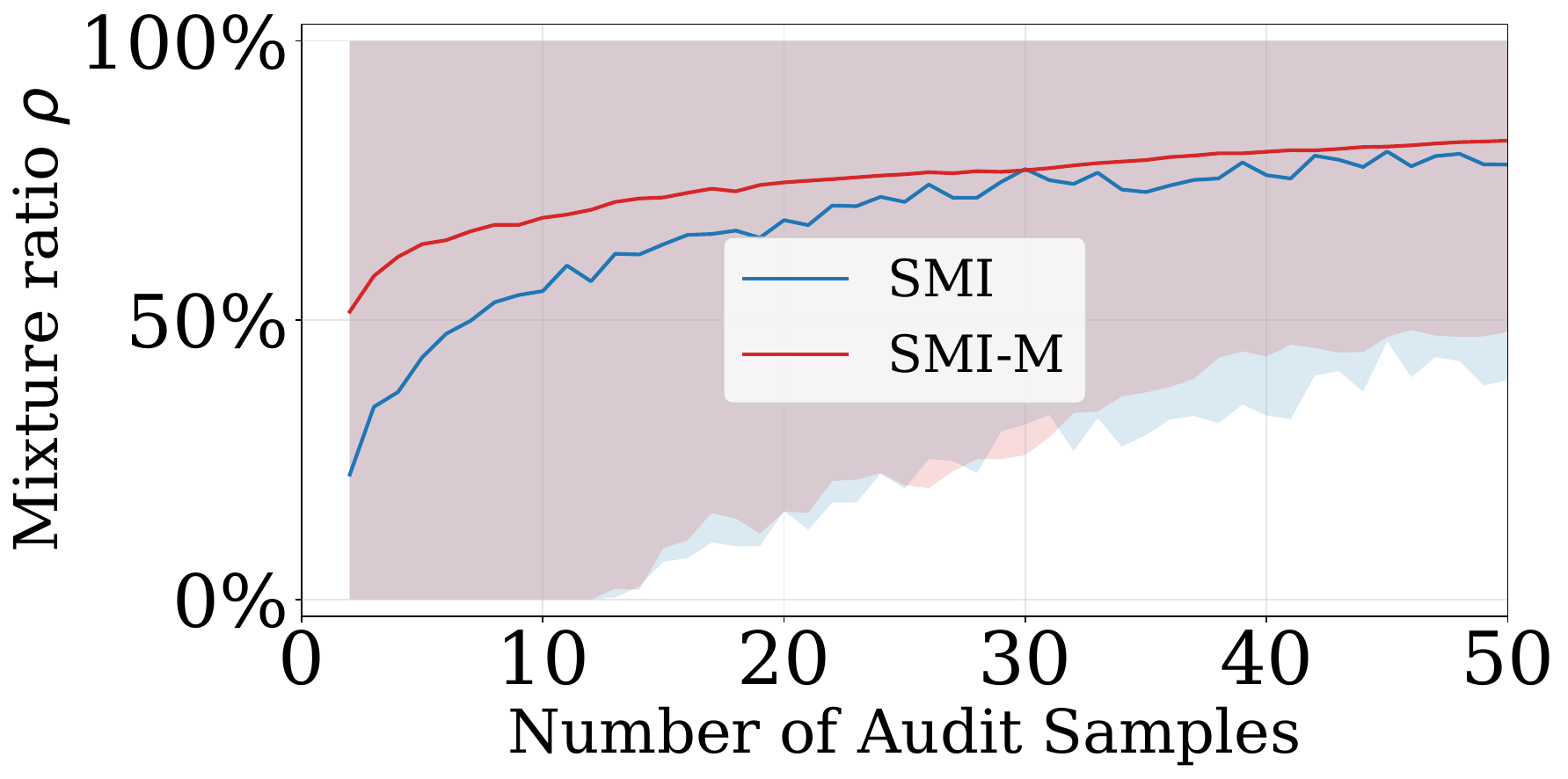}
    \caption{ResNet18-10\%}
    \label{fig:cinic-resnet50-b}
\end{subfigure}
\begin{subfigure}[t]{0.48\linewidth}
    \centering
    \includegraphics[width=\linewidth]{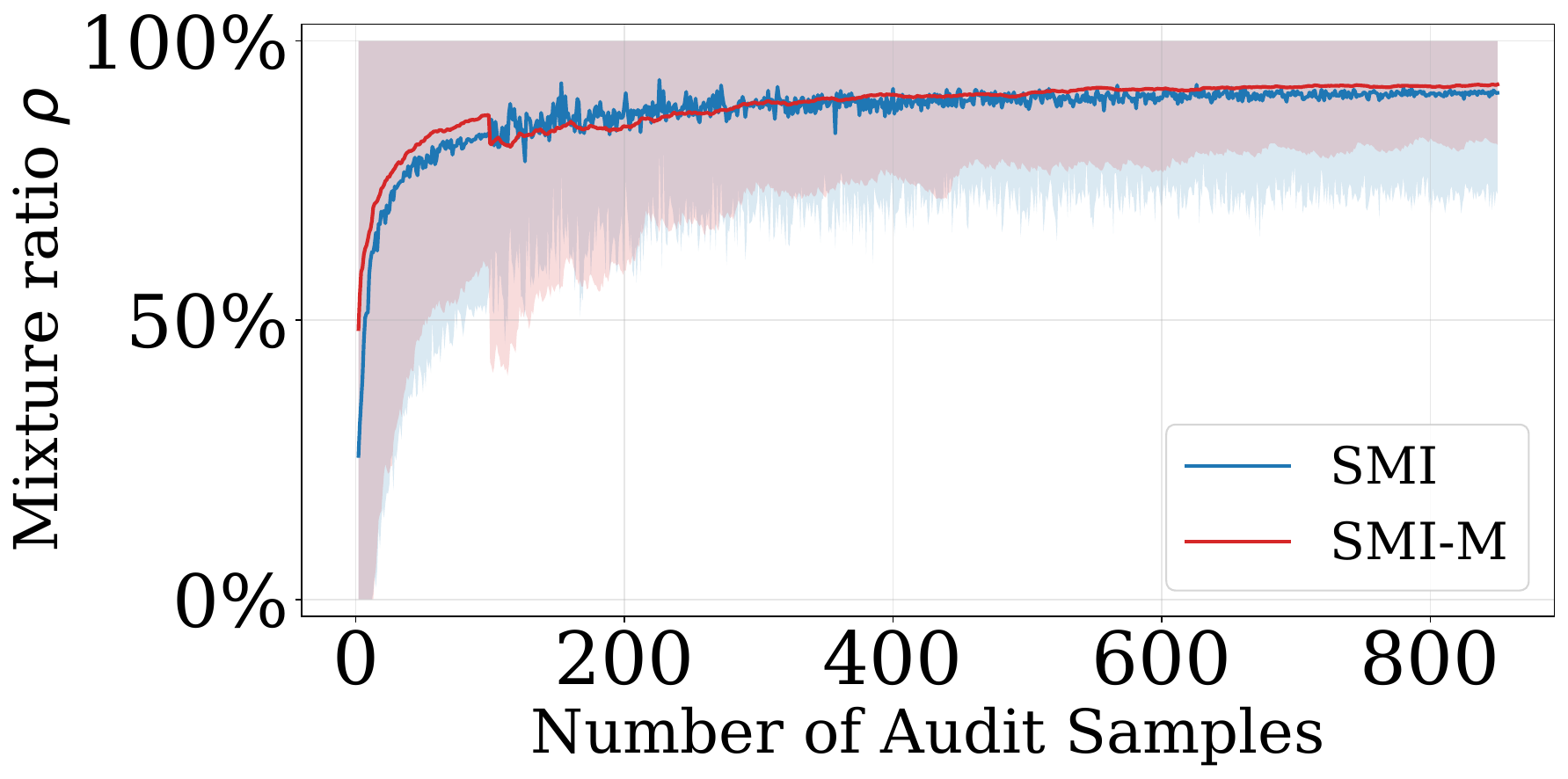}
    \caption{ResNet50-5\%}
    \label{fig:cinic-resnet50-a}
\end{subfigure}
\hfill
\begin{subfigure}[t]{0.48\linewidth}
    \centering
    \includegraphics[width=\linewidth]{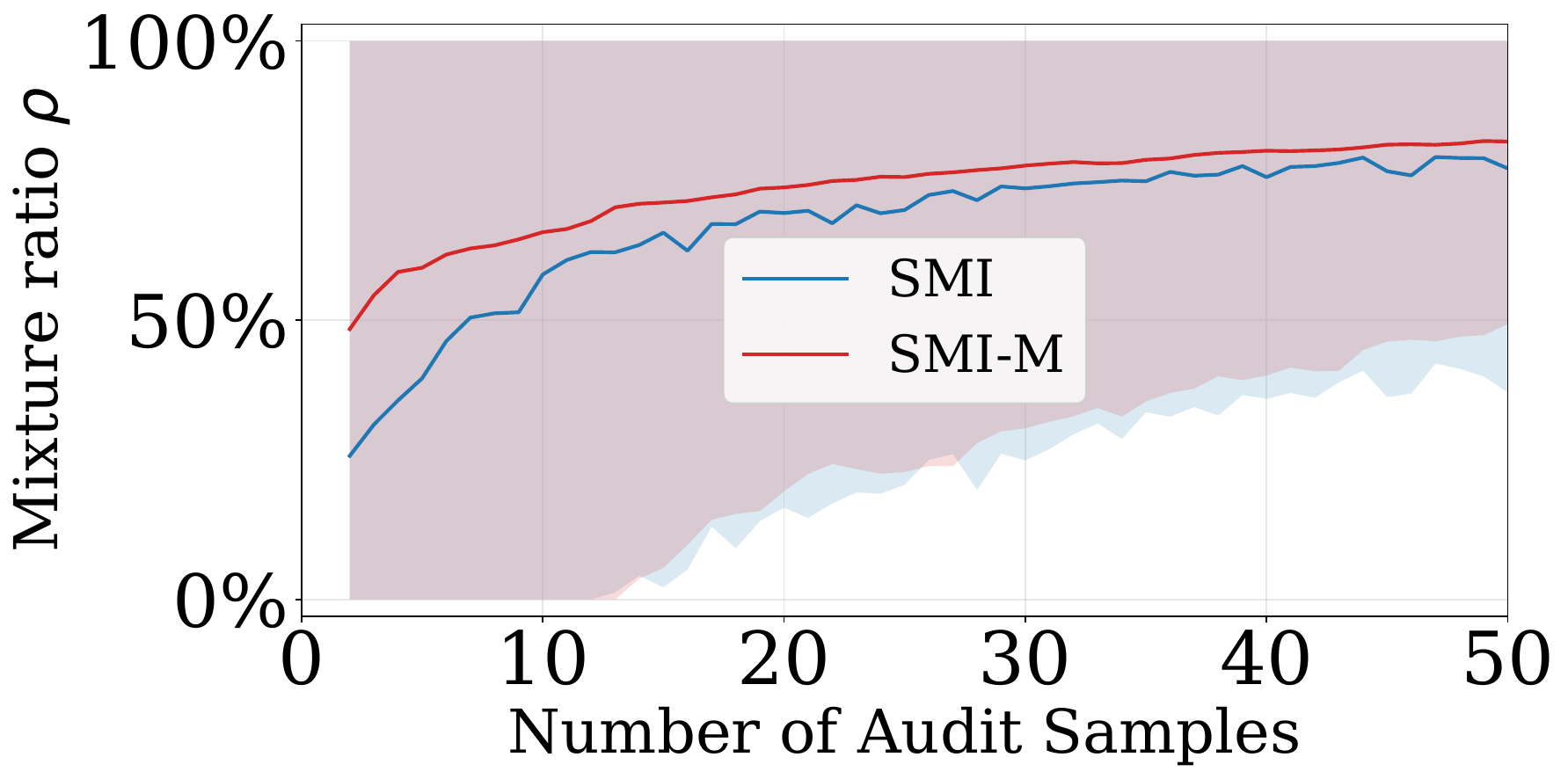}
    \caption{ResNet50-5\%}
    \label{fig:cinic-resnet50-b}
\end{subfigure}
\begin{subfigure}[t]{0.48\linewidth}
    \centering
    \includegraphics[width=\linewidth]{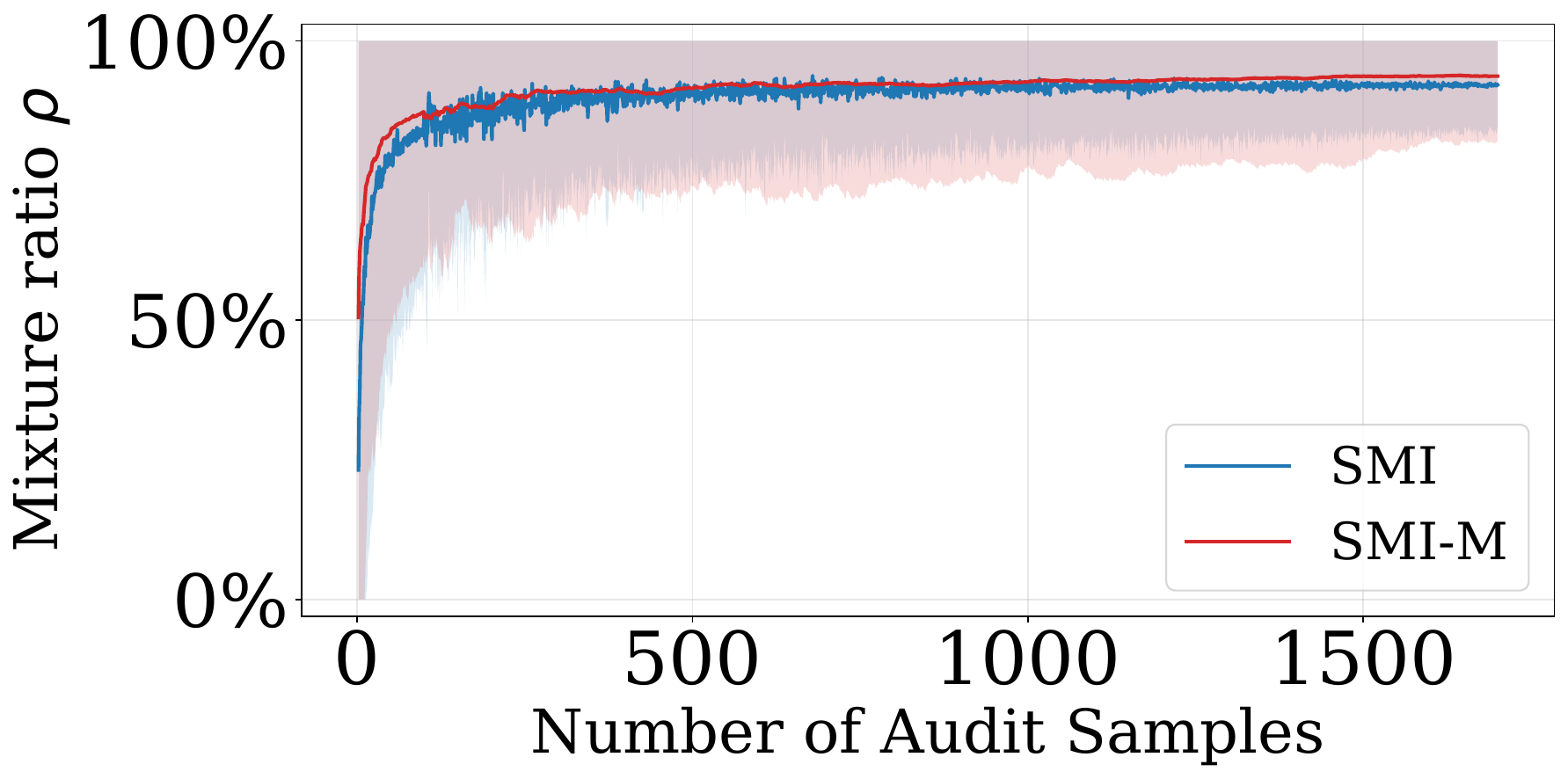}
    \caption{ResNet50-10\%}
    \label{fig:cinic-resnet50-a}
\end{subfigure}
\hfill
\begin{subfigure}[t]{0.48\linewidth}
    \centering
    \includegraphics[width=\linewidth]{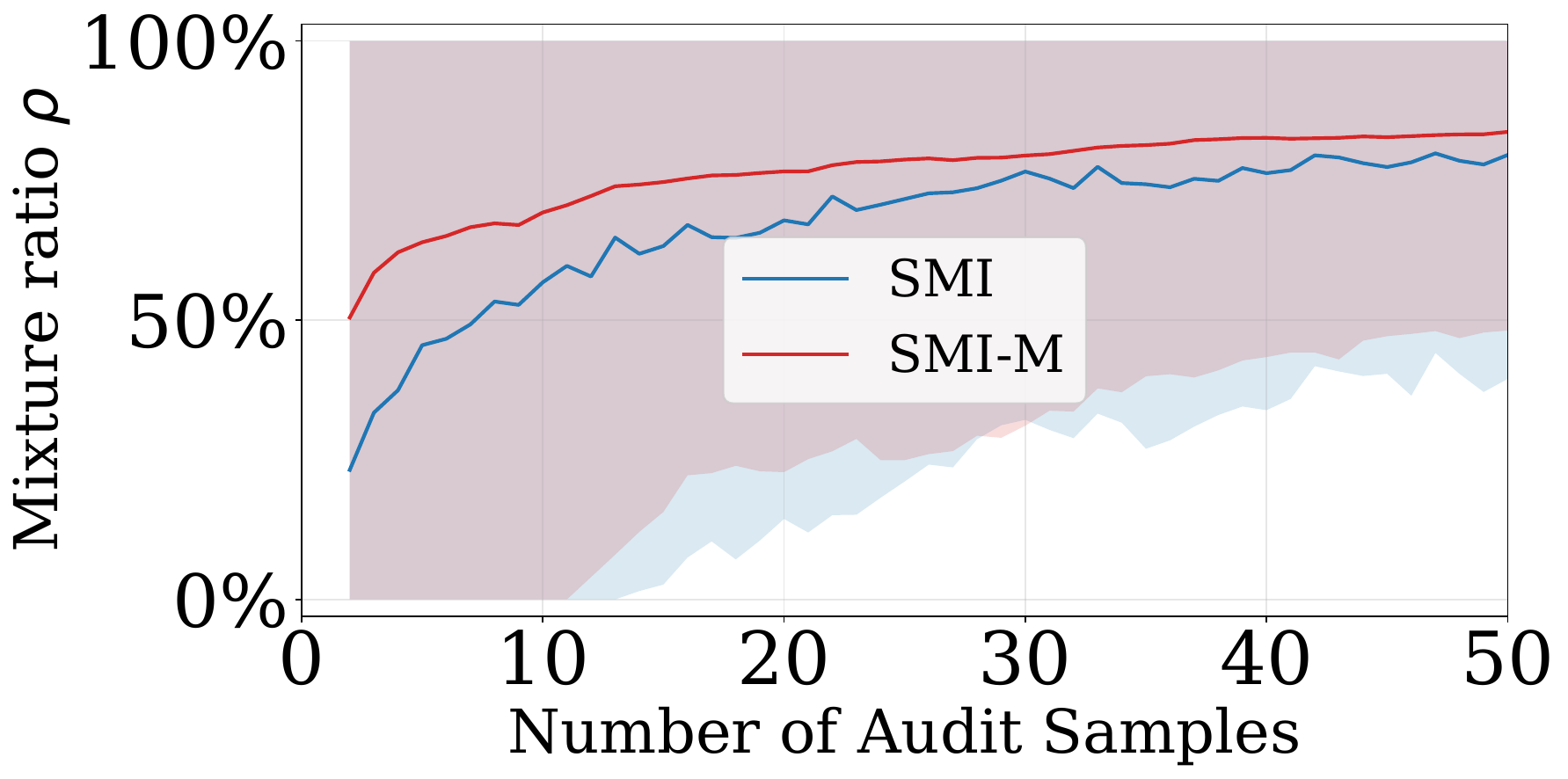}
    \caption{ResNet50-10\%}
    \label{fig:cinic-resnet50-b}
\end{subfigure}
\caption{Performance fluctuation of SMI and SMI-M with different numbers of resampling iterations on 5\% and 10\% randomly unlearned data.}
\label{fig:A2}
\vspace{-1.0em}
\end{figure}
\subsection{Experimental Environment}
We conduct the experiments on eight NVIDIA RTX A6000 GPUs, with Python 3.11 as the experimental environment.
\subsection{Unlearning methods} \label{unlearning}
We select Retrain, FineTune, Fisher, and Scrub as baseline methods for machine unlearning, in order to compare different unlearning strategies in terms of forgetting effectiveness and retained performance. Let $D_u$ denote the dataset to be unlearned, and let $D_m$ denote the member dataset after removing $D_u$.

\textbf{Retrain} is the most direct unlearning baseline. It removes the unlearning dataset $D_u$ from the original training set and trains a model with the same architecture from scratch using only the member dataset $D_m$. Since this model never observes $D_u$ during training, Retrain is usually regarded as the ideal reference for machine unlearning. However, it requires a full retraining process and therefore incurs substantial computational cost, so it is mainly used to evaluate the effectiveness of other unlearning methods.

\textbf{FineTune} adopts a fine-tuning-based unlearning strategy. Starting from the original model, it continues training the model using only the member dataset $D_m$, so that the model gradually shifts toward a state that relies more on the retained data. This method is simple to implement, has lower computational cost than retraining, and can weaken the influence of the unlearning data $D_u$ to some extent. However, since the model parameters are inherited from the original model, FineTune cannot guarantee complete removal of the information contained in $D_u$.

\textbf{Fisher} adopts a Fisher-information-based unlearning method. It estimates the importance of model parameters to determine which parameters are more critical for model predictions, and then adjusts or perturbs the parameters accordingly to weaken the influence of the unlearning data $D_u$. The core idea of this method is to exploit local second-order information in the parameter space, so as to forget $D_u$ while preserving the model performance on $D_m$ as much as possible. Compared with simple fine-tuning, Fisher pays more attention to the differences in how different parameters contribute to the model behavior.

\textbf{Scrub} adopts a distillation-based unlearning method. It guides the model to deviate from the original predictions on the unlearning dataset $D_u$, while preserving the prediction behavior of the original model on the member dataset $D_m$. In this way, Scrub balances forgetting effectiveness and retained performance: it weakens the model's memorization of the target unlearning samples while minimizing performance degradation on the remaining data.

\end{document}